\documentclass[journal]{IEEEtran}

\usepackage{multirow}
\usepackage[table]{xcolor}
\usepackage{amsmath}
\usepackage{gensymb}
\usepackage{amsthm}
\usepackage{verbatim}
\usepackage{subfigure,graphicx}
\usepackage[outdir=./Figs/]{epstopdf}
\usepackage{amssymb,bm,color}
\usepackage[citecolor=red,urlcolor=red]{hyperref}
\newtheorem{remark}{Remark}
\newtheorem{theorem}{Theorem}
\newtheorem{assumption}{Assumption}
\usepackage[capitalize]{cleveref}
\hypersetup{
	colorlinks=true,
	linkcolor=blue,
	filecolor=red,      
	urlcolor=blue,
	citecolor=red,
}
\usepackage{arydshln} 
\usepackage{booktabs}
\usepackage[table]{xcolor}
\usepackage{mathrsfs}
\usepackage{algpseudocode}
\usepackage{algorithm}
\usepackage{mathtools, cuted}
\usepackage{lipsum, color}
\usepackage{multicol}
\usepackage{makecell}
\usepackage{booktabs}
\usepackage{cite}
\usepackage{diagbox}
\usepackage[table ]{ xcolor}
\usepackage{svg} 


\newtheorem{lemma}{Lemma}

\newtheorem{defn}{Definition}

\newcommand{\T}{^\top}
\newcommand{\C}{\mathcal}

\usepackage{soul}
\usepackage{soul}
\usepackage{amsmath}
\usepackage[table ]{ xcolor}
\usepackage{stfloats}
\soulregister\cite7 %
\soulregister\citep7 %
\soulregister\citet7 %
\soulregister\ref7 %
\soulregister\pageref7 %
\graphicspath{{Figs/}}
\usepackage{graphicx} 
\usepackage{epstopdf}
\usepackage{float}  %
\usepackage{subfigure}  %
\usepackage{threeparttable}
\usepackage{svg}
\usepackage{orcidlink}
\usepackage{cases}
\usepackage[table]{xcolor}
\usepackage{adjustbox}

\begin{document}
	
	\title{%

%
%
%
%
		
		Aerial GRIPPER: A Gradient-based Real-time Inverse-game Predictor and Planner
		
	}
	
	\author{Zeshuai Chen, Meng Wang, Jindou Jia, Xiang Yu$^{*}$, and Lei Guo,~\IEEEmembership{Fellow,~IEEE}
		
		\thanks{This work was supported in part by the National Natural Science Foundation of China (Grant Numbers 62425302, 62595803, 62388101), in part by Fundamental and Interdisciplinary Disciplines Breakthrough Plan of the Ministry of Education of China (Grant Number JYB2025XDXM206), in part by "Pioneer" and "Leading Goose" R\&D Program of Zhejiang (Grant Number 2025C01053). (\emph{Corresponding author: Xiang Yu}.)}
		
		\thanks{Zeshuai Chen, Meng Wang, Xiang Yu, and Lei Guo are with School of Automation Science and Electrical Engineering, Beihang University, 100191, Beijing, China (e-mail: mwangbuaa@126.com; zschenbuaa@126.com; xiangyu$\_$buaa@buaa.edu.cn; lguo@buaa.edu.cn).}
		
		\thanks{Jindou Jia is with School of Mechanical and Aerospace Engineering, Nanyang Technological University, 639798, Singapore (e-mail: jindou.jia@ntu.edu.sg).}
		
	}
	
	\maketitle
	
	\IEEEpubid{%
		\begin{minipage}{\textwidth}
			\raggedright
			\color{gray}
			\fontsize{6.5pt}{7.8pt}\selectfont
			\copyright~2026 IEEE.  Personal use of this material is permitted.  Permission from IEEE must be obtained for all other uses, in any current or future media, including reprinting/republishing this material for advertising or promotional purposes, creating new collective works, for resale or redistribution to servers or lists, or reuse of any copyrighted component of this work in other works.
		\end{minipage}
	}

	\begin{abstract}
		
		Accurate capture of non-cooperative targets is critical. In an attempt to tackle this intractable challenge, an aerial gripper system integrated with a Gradient-based Real-time Inverse-game Predictor and PlannER (GRIPPER)  framework is proposed. The interaction is formulated as a general-sum pursuit-evasion game under incomplete information. Specifically, underlying cost parameters of the target are inferred online, and the open-loop Nash equilibrium (OLNE) strategy is iteratively refined within a receding-horizon loop. To ensure high-frequency execution, a computationally friendly gradient-based inverse-game solver is developed. Without explicit computation of the Hessian inverse, the optimized solution is updated ($>$ 50 Hz) based on implicit differentiation and fast Hessian-vector products. Meanwhile, an anti-disturbance controller is developed to overcome disturbances of uncertain payload and gripper actuation, enabling precise tracking of the planned trajectory and accurate grasping of the target. Simulations and real-world experiments illustrate the superior computational efficiency and task performance of GRIPPER. The task of capturing and delivering a non-cooperative target is accomplished, highlighting the robustness, adaptability, and real-time performance of the framework in highly adversarial scenarios.
		
	\end{abstract}

	\begin{IEEEkeywords}
		Aerial capture, non-cooperative target, game-theoretic planning, inverse game solver
	\end{IEEEkeywords}
	
	\IEEEpubidadjcol
	
	\section{Introduction}
	
	\IEEEPARstart{A}{erial} manipulators, composed of a multi-rotor unmanned aerial vehicle (UAV) and a dexterous manipulator, have revolutionized the ability of aerial vehicles to interact dynamically with the environment\cite{bonyankhamsehAerialManipulationLiterature2018a, olleroPresentFutureAerial2022, liuCompactAerialManipulator2024}. Compared to ground mobile manipulators, aerial manipulators can cross areas where humans or ground robots struggle, greatly expanding the scope of interactive tasks. The combination of manipulability and maneuverability enables aerial manipulators to tackle more complex interactive tasks, such as object grasping\cite{wangPreciseEndeffectorControl2025}, pipeline inspection\cite{bodieActiveInteractionForce2021a}, door opening\cite{leeAerialManipulationUsing2020}, and peg-in-hole tasks\cite{wangMillimeterlevelPickPeginhole2024}.

	Few studies have demonstrated the capability of dynamically grasping non-cooperative objects. Different from static object grasping\cite{wangMillimeterlevelPickPeginhole2024, byunHybridControllerEnhancing2024, mengDesignImplementationRotor2016}, the capture of non-cooperative targets poses significant challenges due to the maneuverability, uncertainty, and adversarial evasion of the target. The composition of real-time trajectory prediction and precise target capture lies beyond the scope of traditional motion planning paradigms. Biological systems exhibit the ability to adapt their own behaviors based on predictions of an opponent’s actions. For instance, bees perform only simple target tracking during foraging. In contrast, dragonflies can anticipate the future trajectory of flying insects and execute capture maneuvers with millimeter-level accuracy\cite{lancerPreattentiveFacilitationTarget2022, mischiatiInternalModelsDirect2015}. Such biological evidence highlights the importance of integrating real-time intent inference with dynamic trajectory planning in robotic aerial manipulations.

	Most of the existing studies focus on offline trajectory optimization\cite{luoTimeoptimalHandoverTrajectory2023, ubellackerHighspeedAerialGrasping2024, chenOnlineTrajectoryGeneration2023}. For example, in \cite{luoTimeoptimalHandoverTrajectory2023}, a moving object along a known predefined trajectory is captured by the aerial manipulator. However, such a known prior assumption may no longer be feasible, due to the fact that non-cooperative objects will actively deviate from their predefined trajectory in response to the capture of aerial manipulator. Additionally, the challenging capture task is simplified by decoupling the prediction and planning phases\cite{ubellackerHighspeedAerialGrasping2024, chenOnlineTrajectoryGeneration2023}. The motion of the non-cooperative target is estimated and subsequently used as an input to the planning module. Nevertheless, the formulation explicitly neglects the complexity induced by the non-cooperative nature. Consequently, these methods cannot adapt in real time to the spontaneous movements of the target, limiting their applicability in dynamic scenarios.
	
	\IEEEpubidadjcol
	
	An additional challenge stems from the incomplete information regarding strategic intent of the non-cooperative target. In highly adversarial scenarios, it is nontrivial to infer the hidden parameters governing its evasion strategy, particularly its real-time trade-offs between energy efficiency, path adherence, and escape urgency. Although the state information can be estimated through probabilistic filters\cite{ammounRealTimeTrajectory2009, jinSwitchedKalmanfilter2015} and interactive multiple models (IMM) filters\cite{kaempchenIMMObjectTracking2004, lefkopoulosInteractionawareMotionPrediction2021}, the motion intent of non-cooperative objects remains strategically un-inferable. Recent advances in game-theoretic planning\cite{zhuangPursuitevasionGameMultiple2023, xuMultiplayerPursuitevasionDifferential2022, wangGametheoreticPlanningSelfdriving2021, spicaRealtimeGameTheoretic2020, maHierarchicalReinforcementLearning2024, linigerNoncooperativeGameApproach2020, liGamebasedApproximateOptimal2023} offer a viable means to model such interactions. Players in the game are assumed to act rationally, with each player seeking to optimize their individual objectives. Nevertheless, a strict premise requiring full knowledge of the target’s cost function must be held. Furthermore, traditional game theory methods, such as the Hamilton-Jacobi-Isaacs (HJI) and Hamilton-Jacobi-Bellman (HJB) approaches\cite{kimHamiltonJacobiBellman2000, kunduPolicyIterationHamilton2024, liOptimalStrategiesPursuitevasion2023, rutquistSolvingHamiltonjacobibellmanEquation2014, tassaLeastSquaresSolutions2007}, are incapable of dealing with problems involving incomplete information. Therefore, a unified framework is required to simultaneously estimate latent objectives of the target and optimize pursuit strategies on-the-fly.
	
	Lastly, unknown disturbances can severely affect the capture performance, particularly for underactuated UAVs. Numerous studies are presented to achieve the high-precision performance of the end-effector, such as structure design\cite{zhangAerialAdditiveManufacturing2022, dengWholebodyIntegratedMotion2025} and precise aerial platform control\cite{liuDDPGbasedAdaptiveRobust2022, zhangRobustControlAerial2020}. However, unlike static scenarios, the window phase of dynamic capture is extremely limited. Unknown payload and coupling disturbances between the UAV platform and the gripper can pose a negative impact on the capture task. Disturbances encountered during the task must be handled promptly to ensure timely and high-precision capture.
	
	\begin{figure*}
		\setlength{\abovecaptionskip}{-0.3cm}
		\centering{\includegraphics[width=1\textwidth]{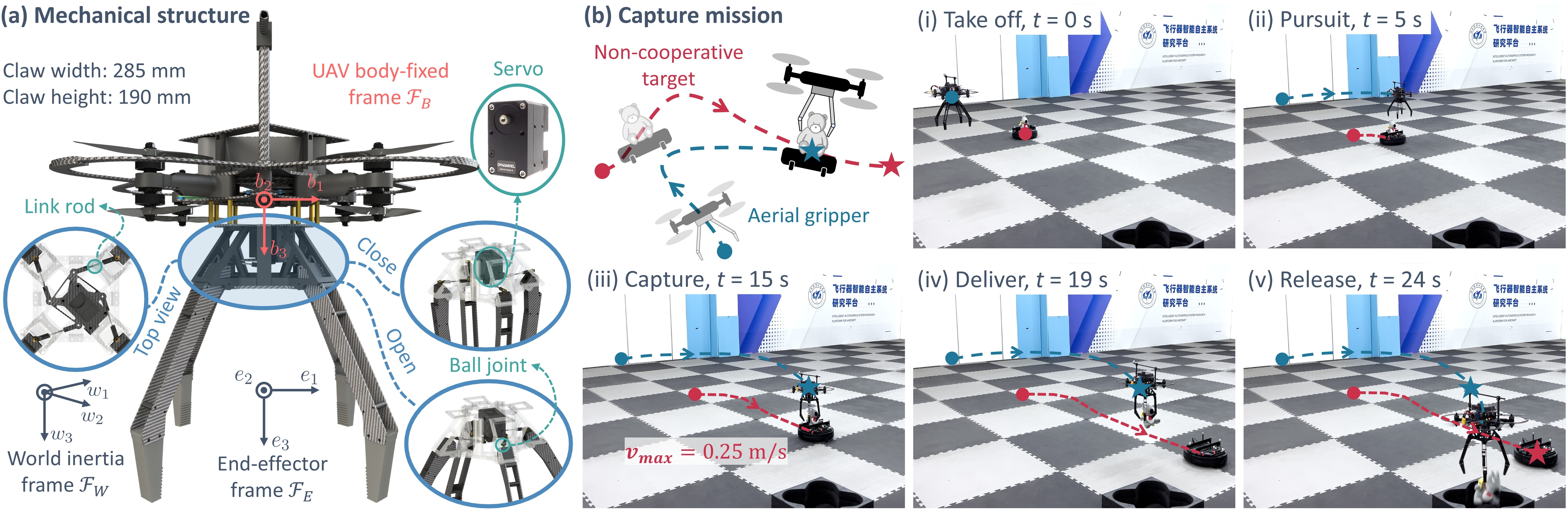}}
		\caption{A custom-made aerial gripper is used to capture a non-cooperative target. \textbf{(a)} Coordinate definition and mechanical structure of the proposed aerial gripper. \textbf{(b)} Illustration and key snapshots of the non-cooperative target capture task. The experimental video can be found in the attached file.}
		\label{Scenario1}
	\end{figure*}
	
	In order to address aforementioned challenges, a \textbf{G}radient-based \textbf{R}eal-time \textbf{I}nverse-game \textbf{P}redictor and \textbf{P}lann\textbf{ER} (GRIPPER) is presented in this work. The centimeter-level trajectory prediction accuracy and real-time dynamic path planning are achieved by developing an inverse-game solver. Furthermore, a fixed-time sliding mode observer (SMO) based controller is adopted to guarantee accurate tracking of the planned trajectory. The main contributions are summarized as follows.
	
	\begin{enumerate}
		
		\item An aerial gripper system integrated with an inverse-game framework is proposed for capturing non-cooperative moving targets. The framework enables online inference of the target intent, prediction of its future trajectory, and computation of  Nash equilibrium strategy. As a result, the system can proactively account for evasive maneuvers and successfully accomplish the challenging task of aerial grasping of highly dynamic targets.
		
		\item To bridge the gap between complex game-theoretic reasoning and real-time requirements of computationally constrained tasks, a lightweight inverse-game solver is proposed. By leveraging implicit differentiation and fast Hessian-vector products (HVPs), the gradient of the upper-level optimization objective with respect to the lower-level parameters is efficiently computed. In contrast to prior works\cite{petersOnlineOfflineLearning2023, xiaoLearningMultipursuitEvasion2024}, the proposed method achieves superior computational efficiency and adaptability for non-cooperative target capture tasks.
		
		\item With respect to payload variation and gripper actuation, a fixed-time SMO based controller is designed to handle composite disturbances. In comparison to disturbance observer (DO) based controller\cite{chenNonlinearDisturbanceObserver2000}, the presented controller guarantees rapid disturbance rejection within a fixed convergence time. This design is particularly critical in short capture window phase, where real-time disturbance mitigation is indispensable to maintain high-precision tracking.
		
	\end{enumerate}
	
	A set of comparative simulations are conducted to validate the superior performance of the proposed framework. Furthermore, a real-world non-cooperative target capture and delivery task is accomplished, demonstrating the real-time performance and task capability of the system.
	
	The rest of the paper is organized as follows. Related work is depicted in Section~\ref{sec:Related Work}. The incomplete information dynamic game framework and the dynamics of the aerial gripper are formalized in Section~\ref{sec:Preliminaries and Problem Formulation}. Section~\ref{sec:Game-Based Trajectory Planning and Prediction} goes into details on the proposed framework. The control strategy of the aerial gripper is introduced in Section~\ref{sec:ControllerDesign}. Comparative simulation and experimental results are presented in Section~\ref{sec:Simulation and Experiments}. Section~\ref{sec:Conclusion} concludes the study. Finally, potential directions for future work are discussed in Section~\ref{sec:discussion}.

	\section{Related Work}
	\label{sec:Related Work} 
	
	This section reviews the relative studies, including game-theoretic path planning, inverse game and bilevel optimization, trajectory prediction of the non-cooperative target, and precise control for aerial manipulator.
	
	\subsection{Game-Theoretic Planning and Pursuit-Evasion Game}
	
	Dynamic games have become a prevalent modeling tool in various applications, such as autonomous driving\cite{cleachALGAMESFastSolver2020, fridovich-keilEfficientIterativeLinearquadratic2020} and competitive vehicle racing\cite{spicaRealtimeGameTheoretic2020, wangGametheoreticPlanningSelfdriving2021}. The Nash equilibrium assumes no hierarchy among players, where the strategy of each player is the best response to that of others. This equilibrium framework has been extensively explored\cite{fridovich-keilEfficientIterativeLinearquadratic2020, spicaRealtimeGameTheoretic2020, chenCooperativeControlPower2015} to reveal strategic interactions in non-cooperative multi-agent systems. The game-theoretic based planners are applied for car and quadrotor racing under different scenarios in \cite{linigerNoncooperativeGameApproach2020, wangGametheoreticPlanningSelfdriving2021, spicaRealtimeGameTheoretic2020}. Results illustrate that the approach outperforms model predictive control (MPC) algorithm in competitive settings.
	
	Inspired by predator-prey interactions in nature, the non-cooperative target capture problem is typically modeled as  a pursuit-evasion game. In such games, a pursuer aims to capture an evader who attempts to avoid being captured. In\cite{imadoMethodSolveMissileaircraft2005}, a differential game approach is proposed to model the PE engagement between a missile and an aircraft. Building upon similar foundations, the optimal control strategy for fixed-time zero-sum PE game is investigated to achieve effective satellite interception\cite{zhangFixedtimeZerosumPursuit2024}. Moreover, a deep reinforcement learning method is leveraged to train agents in PE game, facilitating robust UAV navigation under adversarial pursuit \cite{xiaoLearningMultipursuitEvasion2024}. However, the strong assumption is required that the behavior can be modeled as zero-sum games, overlooking the diversity and variability in the behavioral patterns of opponents.
	
	In our work, the PE game is reformulated as a general-sum game, with an open-loop Nash equilibrium solver developed to generate the desired trajectory for the aerial gripper. Unlike previous works where the objective function of each agent is known\cite{spicaRealtimeGameTheoretic2020, wangGametheoreticPlanningSelfdriving2021, linigerNoncooperativeGameApproach2020, xiaoLearningMultipursuitEvasion2024}, the online estimation of the objective functions of other agents is achieved in the presented framework. The ability to operate in real-world scenarios can be improved, especially when the objectives and intentions of non-cooperative agents are not explicitly available.
	
	\subsection{Inverse Game and Bilevel Optimization}
	
	The purpose of inverse-game theory is to infer the underlying cost functions and strategic preferences of agents based on observed behaviors. Typically, inverse games are solved via Bayesian inference \cite{lecleachLUCIDGamesOnlineUnscented2021}, KKT residual minimization \cite{awasthiInverseDifferentialGames2020}, and  equilibrium-constrained maximum likelihood \cite{petersOnlineOfflineLearning2023}. However, these methods struggle with non-convex or implicit constraints, which are sensitive to noise and inaccurate equilibrium  assumptions.
	
	Recently, bilevel optimization (BLO) provides a natural framework to solve the inverse game, where the upper level estimates agent parameters while the lower level solves for game equilibrium. Classical approaches convert bilevel problems into single-level ones \cite{petersOnlineOfflineLearning2023}. However, the reformulation depends on the KKT condition, leading to a large number of additional variables and constraints. An alternative route is the utilization of gradient-based BLO  \cite{hanLearningPlanManeuverable2024, landryDifferentiableBilevelOptimization2021}. While the pioneering work \cite{liuLearningPlayTrajectory2023} applies the implicit function theorem (IFT) to solve inverse games, it relies on explicit dense matrix inversions that incur prohibitive computational costs. A cost inference framework for feedback dynamic games from noisy partial state observations and incomplete trajectories is developed in \cite{liCostInferenceFeedback2023}. However, the formulation focuses on offline inference from fixed trajectory data.
	
	In our work, the implicit gradient is computed via an iterative approach with fast HVP, resulting in a lightweight and highly efficient solution. This approach enables high-frequency online adaptation required for highly dynamic and time-critical tasks like the aerial grasping of non-cooperative targets.
	
	\subsection{Trajectory Prediction of Non-cooperative Target}
	
	Trajectory prediction has emerged as a critical component in autonomous systems, encompassing a wide range of applications from autonomous driving to robotic manipulation. Trajectory prediction methods can be broadly categorized into data-driven and model-based approaches\cite{huangSurveyTrajectorypredictionMethods2022, korbmacherReviewPedestrianTrajectory2022, yangCVaRConstrainedSafetyCooperation2025}. Data-driven methods, including long short-term memory network (LSTM)\cite{zynerRecurrentNeuralNetwork2018}, generative adversarial networks (GAN)\cite{hegdeVehicleTrajectoryPrediction2020}, and Transformers\cite{giuliariTransformerNetworksTrajectory2021}, have attracted substantial research attention due to their ability to model complex interactions between agents. Model-based methods, such as physics-driven models, offer more interpretable predictions \cite{wangVehicleTrajectoryPrediction2020}. However, these methods suffer from high computational costs and heavily rely on passive learning from extensive pre-collected datasets rather than active adaptation.
	
	To overcome the reliance on extensive offline datasets and passive pattern matching, our framework formulates trajectory prediction as an interactive forward game. The real-time adaption with limited prior data is achieved by incorporating interactive intent reasoning between agents.
	
	\subsection{Precise Control for Aerial Manipulation}
	
	Focusing on the challenging task of capturing non-cooperative object, the guarantee of accurate end-effector tracking performance is of importance. Unknown payload and coupling disturbances between the UAV platform and the gripper can severely affect the capture performance of the aerial gripper. In \cite{wangMillimeterlevelPickPeginhole2024}, a learning based disturbance rejection approach is proposed. The reference trajectory of the manipulator is replanned to directly counteract the floating disturbance of the UAV platform, thereby enabling  precise control of the end-effector. To deal with the dynamic uncertainties resulting from the quadrotor, robotic arm, and payload, an adaptive control scheme is presented in \cite{liuDDPGbasedAdaptiveRobust2022} to allow the aerial manipulator to pick objects successfully. However, the aerial platform is required to maintain a quasi-static motion \cite{wangMillimeterlevelPickPeginhole2024, liuDDPGbasedAdaptiveRobust2022}.
	
	In an attempt to dynamically capture the non-cooperative target, a fixed-time SMO is employed to achieve rapid and robust compensation. The involvement of SMO significantly improves the capture success rate.

	\section{Preliminaries and Problem Formulation}
	\label{sec:Preliminaries and Problem Formulation}
	
	In this section, the mechanical architecture of the proposed aerial gripper system is presented. Subsequently, the fundamental components of the non-cooperative target capture problem are introduced, including the game-theoretic framework for PE interactions and dynamics of the aerial gripper.
	
	\subsection{Problem Description and System Overview}
	\label{subsec:system_overview}
	
	\subsubsection{Problem Description} 
	The primary objective is to maneuver an aerial gripper to capture a moving target within a highly dynamic environment. The challenge is twofold: Firstly, the target is actively evading, and its strategic intent is unknown to the pursuer. Online trajectory planning coupled with intent inference is required to be achieved. Secondly, unknown disturbances during the task should be handled promptly for timely and high-precision capture.
	
	\subsubsection{System Overview} 
	
	Fig. \ref{Framework} presents the proposed hierarchical planning and control framework. The architecture consists of two main modules:
	
	\begin{itemize}
		
		\item \textbf{Game-theoretic planner:} The interaction is modeled as an incomplete information trajectory game. By adopting an inverse-game solver, the intent and trajectory of the target is inferred and predicted online, respectively. The Nash equilibrium is simultaneously computed in a receding-horizon loop to generate proactive feasible reference trajectories for the aerial gripper.
		
		\item \textbf{Anti-disturbance controller:} A fixed-time SMO based controller is developed to guarantee precise end-effector tracking. By resorting to the observer, unknown disturbances are rapidly estimated and rejected. Thus, the reference trajectory generated by the game-theoretic planner is executed precisely.
		
	\end{itemize}
	
	\subsection{Design of Aerial Gripper}
	
	A variety of aerial grippers with different grasping mechanisms have been developed to achieve different tasks, such as suction cups\cite{kremerTRIGGERLightweightUniversal2023}, magnets\cite{luoTimeoptimalHandoverTrajectory2023}, soft fingers\cite{ubellackerHighspeedAerialGrasping2024}, and bio-inspired grippers\cite{yangBistableSoftGripper2025, roderickBirdinspiredDynamicGrasping2021}. However, due to the confined size and payload, such grippers often suffer from limited grasping radius and flexibility. To expand manipulation capabilities, more complex aerial manipulators are developed, including serial robotic arms\cite{wangMillimeterlevelPickPeginhole2024} and parallel manipulators\cite{yangWholebodyRealtimeMotion2021}. Additionally, various morphable UAVs have emerged \cite{wuRingrotorNovelRetractable2023, xuBiomimeticMorphingQuadrotor2024, buckiDesignControlMidairreconfigurable2023}. Nevertheless, such improvements in versatility are typically accompanied by increased system complexity, particularly in planning and control. In practice, the optimal choice of gripper depends heavily on the specific application. Given most small-scale UAVs with limited  payloads, lightweight and low-complexity grasping mechanisms remain practical and effective options\cite{mellingerDesignModelingEstimation2011}.
	
	The developed aerial gripper is composed of an X-shaped coaxial multi-rotor UAV and a single DoF four-finger gripper, as illustrated in Fig. \ref{Scenario1}(a). The compact coaxial rotor configuration offers a high thrust-to-weight ratio, which enhances precision and dexterity to achieve higher payload capacity.
	
	The gripper, integrated at the center underside of the UAV, features a novel design. It is actuated by a longitudinally oriented servo actuator with precise angular position control. As illustrated in Fig. \ref{Scenario1}(a), the servo drives four symmetrically arranged link rods connected to the upper end of the fingers via ball joints. The rotational motion of the servo can be translated into a controlled opening and closing of the end-effector through a leverage effect. A distinguishing characteristic of this configuration is the multiple-purpose of the gripper, which incorporates both manipulation and landing functionalities within a unified mechanical architecture. Conventional designs with separate landing supports would induce spatial constraints and potential collision risk during precise manipulation\cite{wangMillimeterlevelPickPeginhole2024, leeCustomizablePowerGrasps2024, yangBistableSoftGripper2025, linUltralightUltrafastPassive2024}. The proposed design mitigates such limitations, enhancing operational flexibility and extending the effective workspace of the aerial gripper. Additionally, the absence of separated landing support facilitates a larger gripper design, thereby increasing the grasping diameter.
	
	Overall, the structural design of the aerial gripper provides a practical balance between maneuverability and manipulation capability. Detailed parameters of the aerial gripper are listed in Table \ref{GripperParameters}.
	
	\begin{table}[!t]
		\centering
		\renewcommand{\arraystretch}{1.2}
		\caption{Physical parameters of the aerial gripper}
		\begin{tabular}{cc}
			\hline
			Parameter &  Value  \\
			\hline 
			UAV mass & 2.30 kg\\
			Gripper mass & 0.37 kg\\
			Maximum payload & 2.50 kg\\
			Minimum grasping diameter & 1.0 cm\\
			Maximum grasping diameter & 28.5 cm\\
			Minimum servo angle & 0$^{\circ}$\\
			Maximum servo angle & 37$^{\circ}$\\
			\hline
		\end{tabular}
		\label{GripperParameters}
	\end{table}
	
	\subsection{Forward and Inverse Game Formulation}
	
	The non-cooperative target capture problem is a PE game with incomplete information. The aerial gripper attempts to capture the non-cooperative target, while the non-cooperative target tries to evade the pursuit while advancing toward a fixed goal position.
	
	Consider a 2-player discrete-time general-sum trajectory game with horizon of $t\in \{1, \cdots, N\}$. The players are denoted by $i \in \{\C{G, T}\}$, where $\C{G}$ and $\C{T}$ represent the pursuer and the evader, respectively. The game consists of following essential components.
	
	\subsubsection{System dynamics}
	
	Given the control input $u_i \in \mathbb{R}^{m_i}$ for each player $i$, the state $x_i \in \mathbb{R}^{n_i}$ evolves according to the the following difference equation:
	\begin{equation}
		x_{i, t+1} = f( x_{i, t}, u_{i, t}).
		\label{GameDynamics}
	\end{equation}
	
	For clarity, denote the joint state and the joint input as
	\begin{equation}
		\begin{aligned}
			&X_i = \begin{bmatrix} x_{i, 1}^{\top} & \cdots & x_{i, N}^{\top} \end{bmatrix}^{\top} \in \mathbb{R}^{N n_i\times 1},\\
			&U_i = \begin{bmatrix} u_{i, 1}^{\top} & \cdots & u_{i, N}^{\top} \end{bmatrix}^{\top} \in \mathbb{R}^{N m_i\times 1}.
		\end{aligned}
	\end{equation}
	
	\subsubsection{Objective function}

	The objective function of player $i$ is encoded by the distinct cost function $\C{L}_i:\mathbb{R}^{(m_i+n_i)\times N}\rightarrow \mathbb{R}$, which in general depends upon the sequence of states and inputs for all players. In this article, the objectives are expressed in time-additive form, represented by a weighted sum of predefined basis functions. This formulation is common across optimal control and reinforcement learning\cite{englertInverseKKTLearning2017, lecleachLUCIDGamesOnlineUnscented2021, liuLearningPlayTrajectory2023, ziebartMaximumEntropyInverse2008}, where such parametric representations enable tractable optimization and policy learning.
	\begin{equation}
		\begin{aligned}
			&\C{L}_i(X_\C{G}, X_\C{T}, U_\C{G}, U_\C{T}; \lambda_i)\\
			=&\sum_{t=1}^{N} \lambda_i^{\top} \Phi_{i, t}(x_{\C{G}, t}, x_{\C{T}, t}, u_{\C{G}, t}, u_{\C{T}, t}),
		\end{aligned}
	\end{equation}
	where $\Phi_{i, t} \in \mathbb{R}^{\C{N}_i}$ are basis functions, $\lambda_i \in \mathbb{R}^{\C{N}_i}$ are positive weights for each cost component, and $\C{N}_i$ is the number of basis functions.
	
	\subsubsection{Nash equilibrium}
	
	Combining above components, each player $i$ in a dynamic game seeks to solve the following optimization problem: for $\forall i \in \{\C{G,T}\}$
	\begin{equation}
		\begin{array}{*{20}{c}}
			{\mathop {\min }\limits_{X_i,U_i} }&\C{L}_i(X_\C{G}, X_\C{T}, U_\C{G}, U_\C{T}; \lambda_i),\\
			{{\rm{s}}{\rm{.t}}{\rm{.}}}&x_{i,t + 1} = f(x_{i,t}, u_{i,t}).
		\end{array}
	\end{equation}
	
	\begin{remark}
		In this work, an unconstrained optimization structure is employed to guarantee computational efficiency and solvability. Relaxing hard constraints into smooth soft penalties within the costs is a widely adopted and practical approach in the field of trajectory planning \cite{wangWindDriverFrameworkHighManeuverability2025, zhouEGOPlannerESDFFreeGradientBased2021, zhouSwarmMicroFlying2022}.
	\end{remark}
	
	In this problem, the feasible set for each player depends on the decision variables of the rival player, which results in an Nash equilibrium problem. Mathematically, the Nash equilibrium is defined as follows.
	
	\begin{defn}[Nash Equilibrium]
		For a 2-player dynamic game, a set of strategies $U_i^*$ with the corresponding states $X_i^*$ constitutes a solution of Nash equilibrium problem if the following inequalities hold
		\begin{equation}
			\begin{aligned}
				&\C{L}_i^* = \C{L}_i(X_i^*, X_j^*, U_i^*, U_j^*) \le \C{L}_i(X_i, X_j^*, U_i, U_j^*),\\
				&\left(X_i, U_i\right) \ne \left(X_i^*, U_i^*\right), \quad i\ne j,\quad \forall i,j\in \{\C{G,T}\},
			\end{aligned}			
		\end{equation}
		which imply that at a Nash equilibrium point, no player can reduce its cost by changing its strategy unilaterally \cite{basarDynamicNoncooperativeGame1998}. 
	\end{defn}
	
	Denote the Nash equilibrium as a mapping of two players' objectives
	\begin{equation}
		(X_i^*, U_i^*) = NE(\C{L_G}, \C{L_T}).
	\end{equation}
	
	\begin{figure*}[!t]
		\centering{\includegraphics[width=0.75\textwidth]{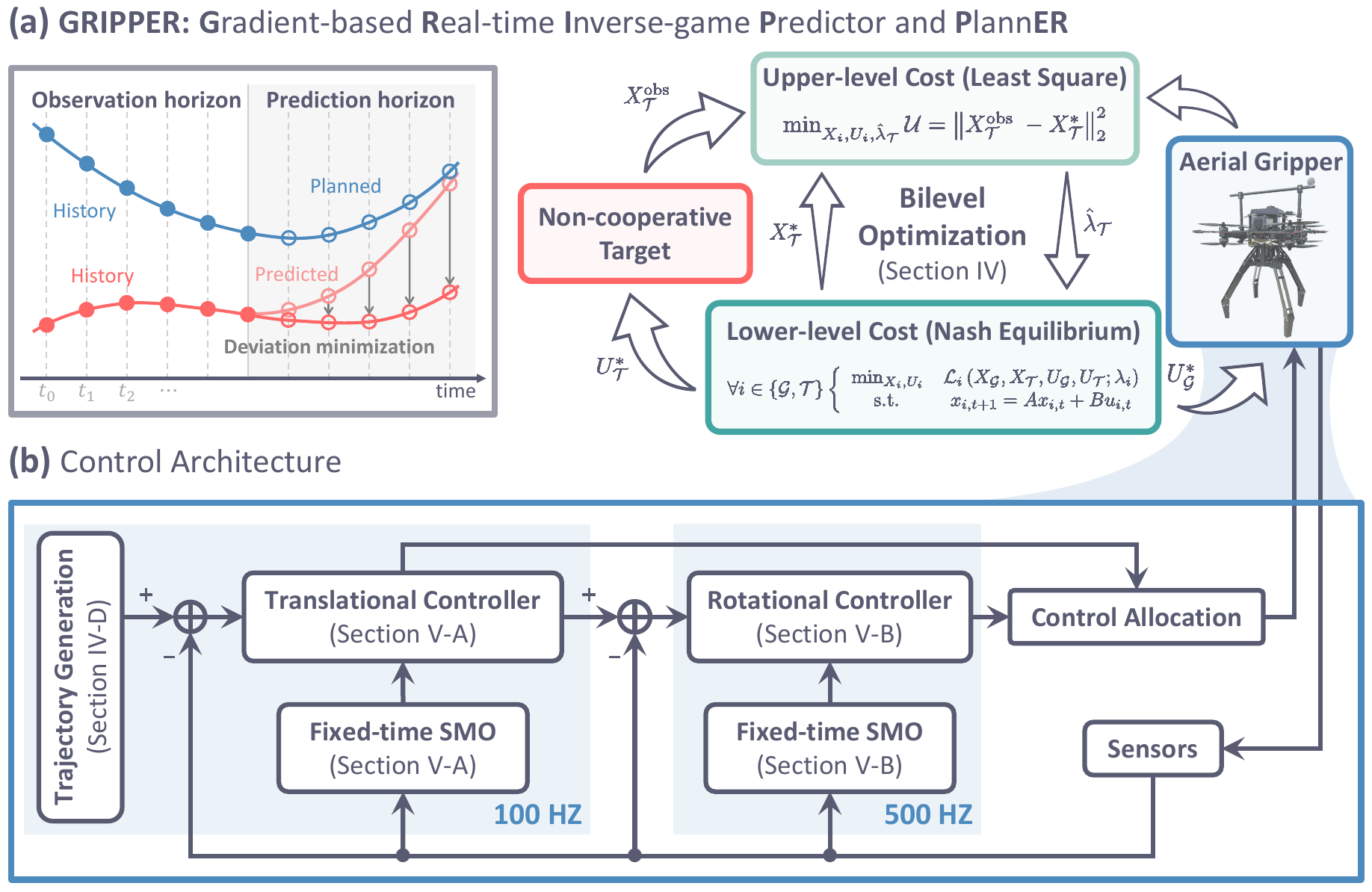}}
		\caption{Proposed game-theoretic planning algorithm and control scheme. \textbf{(a)} Framework of GRIPPER. The incomplete information PE game is formulated as a bilevel optimization framework. The lower-level optimization solves for the Nash equilibrium. The upper level optimization focuses on the estimation of unknown cost weights by minimizing the deviation between the actual trajectory and the Nash trajectory of the non-cooperative target. \textbf{(b)} Control diagram of the aerial gripper. A PD controller is adopted for position control, and a geometric controller is adopted for attitude control. Two fixed-time SMOs are employed to compensate for disturbances in the translational and rotational loops, respectively.}
		\label{Framework}
	\end{figure*}
	
	\subsubsection{Inverse game formulation}
	
	Subsequently, the inverse Nash problem is formulated as one of online parameter estimation problem, with objective functions parameterized by initially unknown values $\lambda_{\C{T}}$. The aerial gripper shall address the non-cooperative game from an ego-centric perspective without prior knowledge of the target objective.
	
	Inspired by \cite{petersOnlineOfflineLearning2023}, a least-square formulation is employed to minimize the deviation between the observed trajectory $X_\C{T}^\text{obs}$ and the predicted Nash equilibrium trajectory $X_{\C{T}}^{*}$. The inverse Nash problem is thus formulated as a bilevel optimization problem. The upper-level optimization focuses on parameter estimation, while the lower-level optimization corresponds to the Nash equilibrium of the PE game. Specifically, the formulation can be expressed as	
	\begin{equation}
		\begin{array}{*{20}{c}}
			{\mathop {\min}\limits_{X_i,U_i,\hat\lambda_{\C{T}}} }&\C{U}=\Vert X_{\C{T}}^{\text{obs}} - X_{\C{T}}^{*} \Vert^{2}_{2},\\
			{{\rm{s}}{\rm{.t}}{\rm{.}}}&
			(X_i^*, U_i^*) = NE(\C{L_G}, \C{L_T}(\hat\lambda_\C{T})),
		\end{array}
	\end{equation}
	where $\hat\lambda_\C{T}$ is the estimation of $\lambda_\C{T}$.
	
	\begin{assumption}
		\label{assump_target}
		The non-cooperative target behaves as a rational decision-maker that computes a Nash equilibrium strategy given full knowledge of the game structure. Furthermore, the target possesses complete knowledge of the objective function of the pursuer.
	\end{assumption}
	
	\begin{remark}
		In Assumption \ref{assump_target}, the non-cooperative target is modeled as a rational and fully informed adversary. Theoretically, the assumption isolates incomplete information to the pursuer side and guarantees equilibrium consistency of the observed trajectories. As a result, the inverse-game inference problem becomes well-posed. In real-world scenarios where agents exhibit bounded rationality or modeling errors, the inferred parameters should be interpreted as a best-fit rational approximation. Notably, sub-optimal evasion by the target generally renders the capture task easier, as the target fails to fully exploit the pursuer’s strategy.
	\end{remark}
	
	\subsection{Aerial Capture Game Formulation}
	
	\subsubsection{System dynamics}
	
	In general, a UAV is modeled as a rigid body actuated by motors, as described in Eq. (\ref{dynamic_model}). However, instead of capturing full real-world dynamics, a simplified point-mass model is adopted for both the aerial gripper and the non-cooperative target. This approximation enables faster optimization and gradient computation. Accordingly, the system dynamics are formulated using a point-mass model, where motion is described by numerical integration of acceleration and velocity in Cartesian space:
	\begin{equation}
		x_{i, t+1} = f(x_{i, t}, u_{i, t}) = A x_{i, t} +  B u_{i, t},
		\label{GameDynamics_AG}
	\end{equation}
	\begin{equation}
		A = \begin{bmatrix}
			\mathbf{I}_{3\times 3} & \mathbf{I}_{3\times 3}\Delta t \\
			\bm{0}_{3\times 3} & \mathbf{I}_{3\times 3}
		\end{bmatrix},
		B = \begin{bmatrix}
			\bm{0}_{3\times 3} \\
			\mathbf{I}_{3\times 3}\Delta t 
		\end{bmatrix},
	\end{equation}
	where $x_i = \begin{bmatrix} p_i^{\top} & v_i^{\top} \end{bmatrix}^{\top} \in \mathbb{R}^{6}$, with $p_i \in \mathbb{R}^3$ and $v_i = \dot{p}_i \in \mathbb{R}^3$ represent the position and velocity of player $i$ in world frame $\C{F}_W$, respectively. Specifically, $p_\C{G}$ should be chosen as the position of the end-effector $p_e$, which is defined in Section \ref{system_model}. The control input is $u_i = a_i = \ddot{p}_i \in \mathbb{R}^{3}$, which represents the acceleration. $\Delta t$ is the control period. $\mathbf{I}$ and $\bm{0}$ represent the identity matrix and the zero matrix, respectively. 
	
	\subsubsection{Objective function}
	
	The objective function for both players mainly consists of the following components:
	
	\begin{enumerate}[]
		
		\item[a.] Pursuit and evasion
		
		During the capture, the aerial gripper minimizes the Euclidean distance to approach the target, while the target maximizes this distance to evade. The adversarial interaction defines the core pursuit-evasion dynamics, which we formalize through the following cost term
		\begin{equation}
			\begin{aligned}
				\Phi_{\C{G},t}^{\text{pur}} &= \Vert p_{\C{G},t} - p_{\C{T},t} \Vert^{2}_{2},\\
				\Phi_{\C{T},t}^{\text{eva}} &= -\Vert p_{\C{G},t} - p_{\C{T},t} \Vert^{2}_{2}.
			\end{aligned}
		\end{equation}
		
		\item[b.] Goal position
		
		The non-cooperative target has a fixed goal point $p_{\C{T}, \text{goal}}\in \mathbb{R}^3$
		\begin{equation}
			\Phi_{\C{T},t}^{\text{goal}} = \Vert p_{\C{T},t} - p_{\C{T},\text{goal}} \Vert^{2}_{2}.
		\end{equation}
		
		\item[c.] Kinematics limitation
		
		To ensure feasible motion, kinematic penalties are imposed on the movement of both players. These constraints include the speed and acceleration limits of both players, which can be expressed by the following terms
		\begin{equation}
			\begin{aligned}
				\Phi_{i,t}^{\text{vel}} &= \Vert v_{i,t}\Vert^{2}_{2},\\
				\Phi_{i,t}^{\text{acc}} &= \Vert a_{i,t}\Vert^{2}_{2}.
			\end{aligned}
		\end{equation}
		
	\end{enumerate}
	
	Overall, the objective functions of the players are formulated as
	\begin{equation}
		\C{L}_{\C{G}} = \sum_{t=1}^{N}(\lambda_{\C{G}, 1} \Phi_{\C{G}, t}^{\text{pur}} + 
		\lambda_{\C{G}, 2} \Phi_{\C{G}, t}^{\text{vel}} + 
		\lambda_{\C{G}, 3} \Phi_{\C{G}, t}^{\text{acc}}),
	\end{equation}
	\begin{equation}
		\C{L}_{\C{T}} = \sum_{t=1}^{N}(\lambda_{\C{T}, 1} \Phi_{\C{G}, t}^{\text{goal}} + 
		\lambda_{\C{T}, 2} \Phi_{\C{G}, t}^{\text{eva}} + 
		\lambda_{\C{T}, 3} \Phi_{\C{G}, t}^{\text{vel}} +
		\lambda_{\C{T}, 4} \Phi_{\C{G}, t}^{\text{acc}}).
	\end{equation}
	
	\subsubsection{Incomplete information}
	In this scenario, we assume the goal position of the target is known. The specific weighting parameter $\lambda_\mathcal{T}$, which determines the trade-off between control effort, goal tracking, and evasion, remains unknown to the pursuer.
	
	\subsection{System Model}
	\label{system_model}
	
	As illustrated in Fig. \ref{Scenario1}(a), three coordinate frames are defined to describe the kinematics and dynamics of the aerial gripper: the world inertia frame $\C{F}_W = [w_1, w_2, w_3]$, the UAV body-fixed frame $\C{F}_B = [b_1, b_2, b_3]$ attached to the center of mass (CoM) of the UAV, and the pose frame attached to the gripper end-effector $\C{F}_E = [e_1, e_2, e_3]$. The frames are defined according to the common North-East-Down convention.
	
	On the basis of the established coordinate frames, the kinematics of the end-effector is given as
	\begin{equation}
		\label{kine_ee}
		\left\{
		\begin{aligned}
			& p_e = p_b + \bm{R}_b p_{\Delta}\\
			& \bm{R}_e = \bm{R}_b \bm{R}_e^b
		\end{aligned}
		\right.
		,
	\end{equation}
	where $p_e \in \mathbb{R}^3$ is the position of the end-effector in $\C{F}_W$, and $p_b\in \mathbb{R}^3$ is the position of the UAV platform in $\C{F}_W$. $\bm{R}_b$, $\bm{R}_e$ and $\bm{R}_e^b \in \mathbb{R}^{3\times 3}$ represent the attitude rotation matrices from $\C{F}_B$ to $\C{F}_W$, $\C{F}_E$ to $\C{F}_W$, and $\C{F}_E$ to $\C{F}_B$, respectively. $p_{\Delta}$ is the constant position deviation between $\C{F}_B$ and $\C{F}_E$. As the frame $\C{F}_E$ is fixed directly below $\C{F}_B$, the rotation matrix $\bm{R}_e^b = \mathbf{I}_{3\times 3}$.
	
	Subsequently, taking derivative of Eq. (\ref{kine_ee}) yields
	\begin{equation}
		\left\{
		\begin{aligned}
			& v_e = v_b + \bm{R}_b \omega_b^{\times}p_{\Delta}\\
			& \omega_e = \omega_b
		\end{aligned}
		\right.
		,
	\end{equation}
	where $v_e \in \mathbb{R}^3$ and $v_b\in \mathbb{R}^3$ represent the velocity of the end-effector and the UAV platform in $\C{F}_W$, respectively. $\omega_e \in \mathbb{R}^3$ denoted the angular velocity of the end-effector in $\C{F}_E$. $\omega_b \in \mathbb{R}^3$ represents the angular velocity of the UAV in the body frame $\C{F}_B$. As a convention, $(\cdot)^{\times}$ denotes the skew-symmetric operator of a vector.
	
	The dynamics of the aerial platform considering the disturbances can be formalized as\cite{zhangSafetyPlanningControl2023}
	\begin{subequations}\label{dynamic_model}
		\begin{numcases}{}
			m \dot v_b = -f_b b_3 + mgw_3 + d_f, \label{trans_dynamics}\\
			J \dot \omega_b = -\omega_b^{\times}J\omega_b + \tau_b + d_{\tau},  \label{rot_dynamics}
		\end{numcases}
	\end{subequations}
	where $m$ is the total mass of the aerial gripper, $g$ is the gravitational acceleration, $J \in \mathbb{R}^{3\times 3}$ represents the inertial matrix in $\C{F}_B$, $f_b \in \mathbb{R}$ stands for the lift generated by eight rotors, and $\tau_b \in \mathbb{R}^3$ represents the generated torque in $\C{F}_B$. Moreover, $d_f \in \mathbb{R}^3$ and $d_{\tau}\in \mathbb{R}^3$ denote the lumped disturbance force and torque, respectively, caused by internal model uncertainty, external disturbances, and coupling disturbances between the UAV platform and the gripper.

	\section{Gradient-based Real-time Inverse-game Predictor and Planner}
	\label{sec:Game-Based Trajectory Planning and Prediction}
	
	In this section, a framework named GRIPPER is presented for solving incomplete information dynamic games, as shown in Fig. \ref{Framework}(a). In the PE game, the GRIPPER framework is employed to estimate the unknown objective weight $\lambda_{\C{T}}$. Based on the estimation, the Nash equilibrium solution is computed. In other words, the behavior of the non-cooperative target is predicted while control decisions for the aerial gripper are simultaneously generated in a receding-horizon loop. The deviation between the predicted state and observations is exploited to update the estimation of $\lambda_{\C{T}}$ online.
	
	\subsection{Nash Equilibrium Solution}
	\label{NashEquil}
	
	Numerous methods and solvers have been developed for this class of problems, with KKT-based methods and their variants being particularly prevalent\cite{petersOnlineOfflineLearning2023}. However, these methods lead to a significant increase in the number of optimization variables and constraints, resulting in substantial computational overhead. In an attempt to improve tractability, all system states and cost functions are reformulated in the strategy space.
	
	Given that the trajectory $X_i$ follows the dynamics described in Eq. (\ref{GameDynamics}), all states and cost functions can be expressed in terms of the initial state $x_{i, 1}$ and the control sequence $U_i$. Thus, we have
	\begin{equation}
		\C{L}_i = \C{L}_i(U_\C{G}, U_\C{T}; x^0; \lambda_i),
	\end{equation}
	where $x^0 = \begin{bmatrix} x_{\C{G},1}^{\top} &  x_{\C{T}, 1}^{\top} \end{bmatrix}^{\top} \in \mathbb{R}^{(n_{\C{G}}+n_{\C{T}})\times 1}$ is the joint initial state. The Nash equilibrium then is simplified to 
	\begin{equation}
		\forall i \in \{\C{G,T}\} , \quad {\mathop {\min }\limits_{U_i} }  \quad   \C{L}_i(U_\C{G}, U_\C{T}; x^0; \lambda_i).
	\end{equation}
	
	In order to solve the above Nash equilibrium problem, we adopt the first-order necessary condition in strategy space that is finite-dimensional differentiable manifold. 
	
	\begin{lemma}
		\label{lem1}
		A local Nash equilibrium satisfies:
		\begin{equation}
			\Omega = \sum_{j=1}^{N}\frac{\partial \C{L_G}}{\partial u_{\C{G}, k}}\mathrm{d}u_{\C{G}, j} + \sum_{k=1}^{N}\frac{\partial \C{L_T}}{\partial u_{\C{T}, k}}\mathrm{d}u_{\C{T}, k} = 0.
			\label{necessary-condition}
		\end{equation}
	\end{lemma}
	
	A detailed explanation of Lemma \ref{lem1} can be referred to \cite{ratliffCharacterizationComputationLocal2013}. The differential game formulation presented here reflects a local, gradient-based view of strategic interaction. $\Omega$ represents the directions in which a player can modify its strategy to most effectively reduce its individual cost. Notably, although the cost of each player depends on the actions of all agents, they can only influence their own cost by unilaterally adjusting their own strategy, which is the core feature of Nash equilibrium computation. This decoupled influence structure aligns with the principle of strategic independence in non-cooperative games. In such games, participants will optimize their actions independently and selfishly, without explicit coordination and cooperation.
	
	As $u_{\C{G}, j}$ and $u_{\C{T}, k}$ are all linearly independent, by accounting for Eq. (\ref{necessary-condition}), the Nash equilibrium condition of the PE game can be expressed as
	\begin{equation}
		\forall i \in \{\C{G,T}\} , \quad \frac{\partial \C{L}_i}{\partial U_i} = \mathbf{0}.
	\end{equation}
	
	Both players have no local incentive to reduce costs by fine-tuning the strategy as all partial derivatives vanish. This formulation makes it easier to reason about the derivatives of upper-level objective with respect to lower-level parameter, which can be leveraged to solve the inverse game problem.
	
	\subsection{Gradient-Based Inverse-Game Solver}
	\label{Inverse}
	
	On the basis of Nash equilibrium formulation, the gradient of upper-level objective $\C{U}$ with respect to lower-level parameter $\lambda_{\C{T}}$ can be derived by the implicit function theorem \cite{liuLearningPlayTrajectory2023}
	\begin{equation}
		\nabla_{\lambda_{\C{T}}} \C{U} =
		\frac{\partial \C{U}}{\partial \lambda_{\C{T}}} +
		\frac{\partial \C{U}}{\partial U_{\C{T}}^*} \frac{\partial U_{\C{T}}^*}{\partial \lambda_{\C{T}}},
		\label{eq23}
	\end{equation}
	where the derivatives $\frac{\partial U_{\C{T}}^*}{\partial \lambda_{\C{T}}}$ are obtained via solving implicit equations from optimality conditions of the lower-level optimization problem. Specifically, taking the derivative of equation $\frac{\mathrm{d}\C{L_T}}{\mathrm{d}U_{\C{T}}^*}=\bm{0}$ with respect to $\lambda_{\C{T}}$ on both sides leads to
	\begin{equation}
		\frac{\partial ^2 \C{L_T}}{\partial U_{\C{T}}^* \partial \lambda_{\C{T}}} +
		\frac{\partial ^2 \C{L_T}}{\partial U_{\C{T}}^{*2}} \frac{\partial U_{\C{T}}^*}{\partial \lambda_{\C{T}}} = \bm{0}.
		\label{eq24}
	\end{equation}
	
	Since $\frac{\partial \C{U}}{\partial \lambda_{\C{T}}}=\bm{0}$, by substituting Eq. (\ref{eq24}) into Eq. (\ref{eq23}), the implicit gradient of $\C{U}$ with respect to $\lambda_{\C{T}}$ can be yielded as
	\begin{equation}
		\nabla_{\lambda_{\C{T}}} \C{U} =
		- \frac{\partial \C{U}}{\partial U_{\C{T}}^*}
		\left(\frac{\partial ^2 \C{L_T}}{\partial U_{\C{T}}^{*2}}\right)^{-1}
		\frac{\partial ^2 \C{L_T}}{\partial U_{\C{T}}^* \partial \lambda_{\C{T}}}.
		\label{ImplicitGradient}
	\end{equation}
	
	However, an additional challenge stems from the involvement of the Hessian inversion due to its high computational and storage costs. Inspired by \cite{wangImperativeLearningSelfsupervised2025}, the Hessian inversion can be computed by solving a special linear equation. For convenience, the following notations are introduced:
	\begin{subequations}
		\begin{numcases}{}
			\vartheta^{\top} = \frac{\partial \C{U}}{\partial U_{\C{T}}^*} \in \mathbb{R}^{1\times 3N}, \label{Notation1}\\
			\bm{H} = \frac{\partial ^2 \C{L_T}}{\partial U_{\C{T}}^{*2}} \in \mathbb{R}^{3N\times 3N},\\
			\bm{H}_{U\lambda} = \frac{\partial ^2 \C{L_T}}{\partial U_{\C{T}}^* \partial \lambda_{\C{T}}} \in \mathbb{R}^{3N\times 4},\\
			\xi^{\top} = \frac{\partial \C{U}}{\partial U_{\C{T}}^*}
			\left(\frac{\partial ^2 \C{L_T}}{\partial U_{\C{T}}^{*2}}\right)^{-1} \in \mathbb{R}^{1\times 3N}.
		\end{numcases}
	\end{subequations}
	
	Further, Eq. (\ref{ImplicitGradient}) can be written as
	\begin{equation}
		\nabla_{\lambda_{\C{T}}} \C{U} = -\xi^{\top} \bm{H}_{U\lambda},
	\end{equation}
	where $\xi$ satisfies $\bm{H}\xi = \vartheta$, which is notably the optimality condition of the following quadratic programming (QP) problem
	\begin{equation}
		{\mathop {\min}\limits_{\xi} } \quad h(\xi) = \frac{1}{2} \xi^{\top}\bm{H}\xi - \xi^{\top} \vartheta.
	\end{equation}
	
	Thus, instead of explicitly calculating and storing the Hessian inversion, $\xi$ can be efficiently obtained by solving this linear system. Specifically, by leveraging the gradient $\frac{\partial h(\xi)}{\partial \xi} = \bm{H}\xi - \vartheta $, $\xi$ can be iteratively derived as
	\begin{equation}
		\hat\xi_{k+1} = \hat\xi_k - \alpha (\bm{H}\hat\xi_k - \vartheta),
		\label{GD1}
	\end{equation}
	where $\alpha \in \mathbb{R}^{3N\times 3N}$ is a positive learning rate matrix, $\hat\xi$ is the estimation of $\xi$. Moreover, the HVP term $\bm{H}\hat\xi_k$ can be calculated without explicit Hessian computation by adopting the fast HVP algorithm
	\begin{equation}
		\begin{aligned}
			\bm{H}\hat\xi_k &= \lim_{\epsilon \to 0}\frac{1}{\epsilon}\left( 
			\left.\frac{\partial \C{L_T}}{\partial U_{\C{T}}}\right|_{U_{\C{T}}^* + \epsilon \hat\xi_k} - 
			\left.\frac{\partial \C{L_T}}{\partial U_{\C{T}}}\right|_{U_{\C{T}}^*} \right)\\
			&=\left.\frac{\partial^{\top}}{\partial U_{\C{T}}}\left(  \frac{\partial \C{L_T}}{\partial U_{\C{T}}}\hat\xi_k \right)\right|_{U_{\C{T}}^*},
		\end{aligned}
		\label{HVP}
	\end{equation}
	where $\frac{\partial \C{L_T}}{\partial U_{\C{T}}}\hat\xi_k$ is a scalar. By this means, the need to compute and store the full Hessian matrix is eliminated. Substantial computational savings can be achieved, especially when dealing with high-dimensional problems.
	
	Similarly to Eq. (\ref{HVP}), the approximated gradient of upper-level cost $\C{U}$ can be expressed as
	\begin{equation}
		\label{HVP2}
		\hat\nabla_{\lambda_{\C{T}}} \C{U} = -\hat\xi^{\top} \bm{H}_{U\lambda}
		= -\frac{\partial^{\top}}{\partial \lambda_{\C{T}}}\left(\frac{\partial \C{L_T}}{\partial U_{\C{T}}}\hat\xi\right).
	\end{equation}
	
	With this approximated gradient, the estimated lower-level parameter $\hat\lambda_{\C{T}}$ can be further updated using the gradient descent method, i.e.
	\begin{equation}
		\hat\lambda_{\C{T}, t+1} = \hat\lambda_{\C{T}, t} - \beta \hat\nabla_{\lambda_{\C{T}}} \C{U},
		\label{GD2}
	\end{equation}
	where $\beta \in \mathbb{R}^{4\times 4}$ is a positive learning rate matrix.
	
	\begin{algorithm}[!t]
		\caption{Gradient-based Real-time Inverse-game Solver}
		\label{algo1}
		\begin{algorithmic}[1]
			\State \textbf{Parameters:} learning rates $\alpha, \beta$
			
			\State \textbf{Input:} unknown parameter $\hat\lambda_T$, Nash equilibrium at prediction step $U^*_\C{T}, U^*_\C{G}$, current observation sequence $X^{obs}_\C{T}$
			
			\State \textbf{Initialize:} intermediate variable $\hat\xi$
			
			\While{not convergent \textbf{and} max step not reached}
			\State Compute $\bm{H}\hat\xi$ and $w$
			\Comment{Eq. (\ref{HVP}), Eq. (\ref{Notation1})}
			\State	$\hat\xi \gets \hat\xi - \alpha ({\bm{H}} \hat\xi - w)$
			\Comment{Eq. (\ref{GD1})}
			\EndWhile
			
			\State Compute $\hat\nabla_{\lambda_{\C{T}}} \C{U}$ 
			\Comment{Eq. (\ref{HVP2})}
			
			\State $\hat\lambda_{\C{T}} \gets \hat\lambda_{\C{T}} - \beta \hat\nabla_{\lambda_{\C{T}}} \C{U}$
			\Comment{Eq. (\ref{GD2})}
			
			\State \textbf{Return} $\hat\lambda_T$
		\end{algorithmic}
	\end{algorithm}
	
	The entire iteration process is conducted online, whereby the estimation of the unknown cost weights is continuously refined. At each time step, the forward game problem is first solved based on the parameter estimation from the previous step. Then, the inverse-game problem is solved based on the Nash equilibrium solution. The complete gradient-descent-based inverse-game solution scheme is summarized in Algorithm \ref{algo1}.
	
	\begin{remark}
		The proposed inverse-game solver can be generalized by re-writing Eq. (\ref{eq24}) to integrate with other differentiable optimization methods, such as KKT-based or NN-based approaches. Specifically, the term ${\partial L}/{\partial U^*}$ can be expressed as ${\partial \tilde{L}}/{\partial U^*}$, where $\tilde{L}$ represents a generic optimality condition of the lower-level optimization problem. For instance, in the case of KKT-based methods, $\tilde{L}$ corresponds to the Lagrangian function. This adaptation allows for the solver to be applied to a wider range of optimization problems beyond the current framework.
	\end{remark}
	
	\begin{algorithm}[!t]
		\caption{GRIPPER: Gradient-based Real-time Inverse-game Predictor and Planner}
		\label{algo2}
		\begin{algorithmic}[1]
			\State \textbf{Initialize:} unknown parameter $\hat\lambda_T$
			
			\If{new observation $x^0$ arrives}
			
			\State Update observation sequence $X^{obs}_\C{T}$
			
			\State $U^*_{\C{G}},U^*_{\C{T}} \gets$ ForwardGameSolver($x^0, \hat\lambda_{\C{T}}$)
			\Comment{Sec. \ref{NashEquil}}
			
			\State $X^{\mathrm{pred}}_{\C{T}} \gets$ TrajectoryPredictor($U^*_{\C{T}}$)
			\Comment{Sec. \ref{TrajPred}}
			
			\State $p_b^d, \dot p_b^d, \ddot p_b^d \gets$ DifferentialFlatnessPlanner($U^*_{\C{G}}$)
			
			\Comment{Sec. \ref{PathPlan}}
			
			\State $\hat\lambda_{\C{T}} \gets$ InverseGameSolver($\hat\lambda_{\C{T}}, X^{obs}_\C{T}, U^*_{\C{G}},U^*_{\C{T}}$)
			
			\Comment{Algorithm \ref{algo1}}
			\EndIf
		\end{algorithmic}
	\end{algorithm}
	
	\subsection{Game-Theoretic Trajectory Prediction}
	\label{TrajPred}
	
	The trajectory prediction serves as a critical component in the target capture task. It enables the aerial gripper to anticipate the future states of the non-cooperative target and plan its approaching motion accordingly. The prediction accuracy is directly dependent on the estimation accuracy of the objective parameter $\hat\lambda_{\C{T}}$, which is continuously refined online through the inverse-game. Furthermore, the predicted trajectory $X_{\C{T}}^{\mathrm{pred}}$ serves as the reference for subsequent prediction error minimization in the inverse game.
	
	The optimal control sequence $U_{\C{T}}^*$ of the non-cooperative target is obtained from the Nash equilibrium solution. With forward propagation of the system dynamics shown as Eq. (\ref{GameDynamics_AG}), the trajectory can be reconstructed as
	\begin{equation}
		\label{Rebuild}
		X_{i}^* = \mathbf{A}x_{i,1} + \mathbf{B}U_{i}^*(\hat\lambda_{\C{T}}),
	\end{equation}
	\begin{equation}
		\mathbf{A} = \begin{bmatrix}
			\mathbf{I}\\
			A\\
			A^2\\
			\vdots\\
			A^{N-1}
		\end{bmatrix},
		\mathbf{B} = \begin{bmatrix}
			\bm{0} \\
			B & \bm{0} \\
			AB & B & \bm{0} \\
			\vdots  & \ddots & \ddots & \ddots\\
			A^{N-2}B & \cdots & AB & B & \bm{0} 
		\end{bmatrix},
	\end{equation}
	where $X_{\C{T}}^{\mathrm{pred}} = X_{\C{T}}^*$ stands for the predicted trajectory of the non-cooperative target, $x_{i,1}$ is the initial state at current time step, and $U_{i}^*$ is the optimal control sequence. The predicted trajectory provides the anticipated path of the target based on its optimal decisions, which can be crucial for tasks such as trajectory planning, interception, or avoidance.
	
	It is worth noting that, the forecast horizon length $N$ plays a significant role in determining the prediction quality. A larger horizon allows for more comprehensive strategic planning but increases computational complexity. The results reported in Table \ref{CalTime_PredErr} are generated from simulated dynamic games under the same configuration as Section \ref{comp_sim} with different forecast horizon $N$. As shown in Table \ref{CalTime_PredErr}, as the forecast horizon increases, the capture time shows a downward trend. Moreover, the correlation between future motion states and current state information weakens as the forecast horizon increases, leading to a decrease in prediction accuracy, as shown in Table \ref{CalTime_PredErr}.
	
	The phenomenon stems primarily from three factors: 1) accumulation of system noise and estimation errors; 2) potential deviation from Nash equilibrium assumptions; 3) unmodeled target behaviors beyond the current observation window. Therefore, in order to mitigate these effects, the prediction horizon of the optimization problem should be reasonably set while considering prediction errors. In our implementation, the horizon length is selected empirically based on the performance curves reported in Table II, balancing computational time and prediction accuracy. The horizon of $N=20$ provides reliable parameter inference while maintaining real-time performance. Additionally, the prediction is continuously updated at each time step as new observations become available and parameter estimation is refined. This receding horizon approach maintains prediction relevance and accuracy throughout the pursuit process.
	
	\subsection{Path Planning Based on Differential Flatness}
	\label{PathPlan}
	
	\begin{table}[!t]
		\centering
		\renewcommand{\arraystretch}{1.2}
		\caption{Prediction performance under different forecast horizon}
		\begin{tabular}{ccccc}
			\hline
			Forecast Horizon  & $N=10$ & $N=15$ & $N=20$ & $N=25$ \\
			\hline 
			Computation Time (ms)  & 13.78 & 14.26  & 19.67 & 35.84	\\
			Prediction Error (mm)  & 1.36  & 3.28  	& 8.07 	& 42.53	\\
			Capture Time (s)       & fail  & fail 	& 9.63 	& 11.29	\\
			\hline
		\end{tabular}
		\label{CalTime_PredErr}
	\end{table}
	
	Given the Nash equilibrium solution from the game-theoretic framework, the optimal control sequence $U_{\C{G}}^*$ for the aerial gripper is obtained. This control sequence represents the desired acceleration of the end-effector. The desired position and velocity trajectory of the end-effector, $X_{\C{G}}^{d} = X_{\C{G}}^*$, can be reconstructed according to Eq. (\ref{Rebuild}). Subsequently, the differential flatness property of the system is leveraged to generate a feasible trajectory for the UAV platform based on the planned trajectory of the end-effector. The control loop is closed by instantiating only the first control value, and periodically reoptimizing the control strategy after updating the state in the next control loop.
	
	Thus, the desired position and velocity of the end-effector can be expressed as:
	\begin{equation}
		\begin{bmatrix}
			p_e^{d}\\
			v_e^{d}\\
		\end{bmatrix} = x^*_{\C{G}, 2},\quad
		a_e^{d} = u^*_{\C{G}, 2}.
	\end{equation}
	
	The desired rotation matrix $\bm{R}_e^{\mathrm{des}}$ can be obtained by exploiting the differential flatness property. The details can be found in \cite{mellingerMinimumSnapTrajectory2011}. By accounting for the kinematic model established in Eq. (\ref{kine_ee}), the desired position of the UAV can be derived as
	\begin{equation}
		p_{b}^{d} = p_{e}^{d} - \bm{R}_{e}^{\mathrm{des}}p_{\Delta},
	\end{equation}
	
	The desired velocity and acceleration can be further derived by taking the derivative of the above equation.
	
	\begin{remark}
		Although the gripper does not have independent degrees of freedom, the proposed framework is centered on end-effector trajectory planning. The planned gripper trajectory is converted into the corresponding UAV reference trajectory, which the UAV tracks to guide the gripper along its planned path.
	\end{remark}

	\section{Control Architecture}
	\label{sec:ControllerDesign}
	
	Precise trajectory tracking and robustness are prerequisites for successful capture of the dynamic target, while unknown disturbances can severely affect the capture performance.
	
	In the aerial gripper system, disturbances primarily originate from two sources. First, the grasping and releasing of an object with unknown mass could significantly deteriorate the control performance. The payload variation not only directly affects the required thrust, but also induces variations in the system inertia. Second, the gripper actuation introduces extra disturbances. The motion of the gripper fingers not only alters the inertia but also changes the position of CoM, resulting in vertical (Z-axis) force disturbances. In addition, servo actuation generates noticeable torque disturbance in the yaw direction. As shown in Fig. \ref{Controller_Exp_test2}(c),(e), experimental results demonstrate that these disturbances are non-negligible.
	
	In non-cooperative target capture tasks, disturbance rejection becomes particularly critical. Unlike static scenarios, the window phase of dynamic capture is extremely limited. Furthermore, rapid recovery of trajectory tracking following disturbances is essential to ensure reliable execution of the game-theoretic planning. Therefore, a fixed-time SMO based controller is presented. Unlike asymptotic observers such as disturbance observer, it can guarantee convergence within a strictly bounded time, regardless of the initial error magnitude. Accurate and timely tracking of the Nash equilibrium trajectory is achieved despite disturbances caused by payload uncertainty and gripper actuation.
	
	The overall control architecture of the aerial gripper is shown in Fig. \ref{Framework}(b). The high-level planner and low-level controller are basically decoupled. The anti-disturbance controller is responsible for faithfully executing the reference trajectory generated by the inverse-game planner. During the short capture window, quickly varying disturbances are primarily handled at the control level. Since the controller bandwidth is substantially higher than that of the planner, short-term disturbances are rejected locally without requiring immediate replanning. The role separation ensures that the closed-loop system maintains end-effector tracking accuracy during the capture phase.
	
	\subsection{Translational Loop}
	
	A cascaded control scheme is employed as shown in Fig. \ref{Framework}(b). The stabilities of the position and attitude loops are considered separately.
	
	The objective of the translational loop is to follow the desired trajectory $p_b^d$ and generate the desired control force $F_b$. Tracking errors of position and velocity are defined as $e_p = p_b^d - p_b$ and $e_v = v_b^d - v_b$. Based on the PD control and acceleration feedforward, the baseline controller of the translational loop is designed as
	\begin{equation}
		\left\{	
		\begin{aligned}
			& a_b^{\text{des}} = K_p e_p + K_v e_v - gw_3 + \ddot p_b^d \\
			& F_0 = m a_b^{des}
		\end{aligned}
		\right.,
		\label{trans_baseline_controller}
	\end{equation}
	where $K_p$, $K_v \in \mathbb{R}^{3\times 3}$ are positive diagonal control gain matrices. By accounting for the disturbance feedforward compensation, the translational anti-disturbance controller can be expressed as
	\begin{equation}
		F_b = F_0 - \hat d_f,
		\label{trans_controller}
	\end{equation}
	where $\hat d_f$ represents the estimation of the disturbance $d_f$ that can be obtained by SMO.
	
	\begin{assumption}
		\label{assump_derivative}
		The derivatives of disturbances represented in Eq. (\ref{dynamic_model}) satisfy $\Vert \dot d_f \Vert \le \rho_1$ and $\Vert \dot d_{\tau} \Vert \le \rho_2$, where $\rho_1$ and $\rho_2$ are unknown positive constants.
	\end{assumption}
	
	\begin{remark}
		In the aerial gripper system, disturbances are composed of both low-frequency components, such as mass and inertia uncertainties, and high-frequency components, such as abrupt payload variations. However, high-frequency signals can be mostly attenuated on account of the inherent low-pass characteristic of the UAV system. Due to the smoothness, continuity, and boundedness of low-frequency disturbances, Assumption \ref{assump_derivative} is considered reasonable. Moreover, although this assumption underlies the theoretical design, the experimental results in Fig. \ref{Controller_Exp}(b),(c) demonstrate that the proposed controller remains effective even under step disturbance induced by object grasping and releasing.
	\end{remark}
	
	On the basis of \cite{yuDesignFixedtimeObserver2018}, by introducing auxiliary variables $\mu_1$, $\varepsilon_1$, and $\varepsilon_2$, the fixed-time SMO can be designed as
	\begin{equation}
		\left\{	
		\begin{aligned}
			& \dot \mu_1 = gw_3 - \frac{f_b}{m}b_3 + \varepsilon_1 \\
			& \varepsilon_1 = -l_1\frac{z_1}{\Vert z_1 \Vert^{1/2}} - l_2 z_1 \Vert z_1 \Vert + \varepsilon_2 \\
			& \dot \varepsilon_2 = -l_3 \frac{z_1}{\Vert z_1 \Vert} \\
			& \hat d_f = m \varepsilon_2
		\end{aligned}
		\right.,
		\label{trans_SMO_controller}
	\end{equation}
	where $z_1 = \mu_1 - v_b$, $l_1$, $l_2$, and $l_3\in \mathbb{R}$ are positive gain scalars. 
	
	The desired thrust $f_b$ can be derived by projecting the desired force vector $F_b$ onto the body-fixed Z-axis, given by $f_b = F_b \cdot (-b_3)$.
	
	\begin{theorem}
		\label{theorem_trans}
		Consider system (\ref{trans_dynamics}) under the translational controller proposed in Eq. (\ref{trans_controller}) and the fixed-time SMO in Eq. (\ref{trans_SMO_controller}). With Assumption \ref{assump_derivative}, all signals in the translational loop are globally uniformly bounded and $\hat d_f$ can converge to $d_f$ uniformly in fixed time, as long as $K_p$ and $K_v$ have all positive diagonal elements and the following conditions hold:
		\begin{equation}
			l_1 > \sqrt{2l_3},l_2>0,l_3>4m^{-1}\rho_1.
		\end{equation}
	\end{theorem}
	
	The detailed stability analysis of the translational loop can be found in Appendix \ref{App_A}.
	
	\subsection{Rotational Loop}
	
	Leveraging differential flatness with Hopf fibration \cite{wattersonControlQuadrotorsUsing2020}, the desired translational acceleration, jerk, and yaw angle are used to compute the desired rotation matrix $\bm{R}_b^d$ and angular velocity $\boldsymbol{\omega}_b^d$. The attitude control loop is thereby designed to track these references and generate the control torque $\boldsymbol{\tau}_b$.
	
	In order to guarantee the control performance of the rotational loop, the rotational control scheme is divided into two parts similar to the translational loop: a baseline controller based on geometric control \cite{jiaFlightControlQuadrotor2022} and a fixed-time SMO. The attitude tracking error and angular velocity tracking error are defined as
	\begin{equation}
		\left\{	
		\begin{aligned}
			e_R & = \frac{1}{2}\left(\bm{R}_{b}\T \bm{R}_{b}^{d} - \bm{R}_{b}^{d\top} \bm{R}_{b}\right)^\vee \\
			e_\omega & = \bm{R}_{b}\T \bm{R}_{b}^{d} \omega_b^d - \omega_b
		\end{aligned}
		\right.,
	\end{equation}
	where $(\cdot)^\vee$ is the inverse of the cross-map. Accordingly, the baseline controller of the rotational loop is designed as:
	\begin{equation}
		\begin{aligned}
			\tau_0 = & K_R e_R + K_\omega e_\omega + \omega_b^{\times}J\omega_b\\
			& - J\left(\omega_b^{\times} \bm{R}_{b}\T \bm{R}_{b}^{d} \omega_b^d + \bm{R}_{b}\T \bm{R}_{b}^d \dot\omega_b^d \right),
		\end{aligned}
		\label{rot_baseline_controller}
	\end{equation}
	where $K_R$, $K_\omega \in \mathbb{R}$ are positive control gains. Thus, the rotational anti-disturbance controller can be expressed as
	\begin{equation}
		\tau_b = \tau_0 - \hat d_{\tau},
		\label{rot_controller}
	\end{equation}
	where $\hat d_{\tau}$ represents the estimation of torque disturbance $d_{\tau}$ obtained by SMO. Similar to the design process of the translational loop observer, the fixed-time SMO in the rotational loop can be formulated as
	\begin{equation}
		\left\{	
		\begin{aligned}
			& \dot \mu_2 = - J^{-1}\omega_b^{\times}J\omega_b + J^{-1}\tau_b + \varepsilon_3 \\
			& \varepsilon_3 = -l_4\frac{z_2}{\Vert z_2 \Vert^{1/2}} - l_5 z_2 \Vert z_2 \Vert + \varepsilon_4 \\
			& \dot \varepsilon_4 = -l_6 \frac{z_2}{\Vert z_2 \Vert} \\
			& \hat d_{\tau} = J \varepsilon_4
		\end{aligned}
		\right.,
		\label{rot_SMO_controller}
	\end{equation}
	where $z_2 = \mu_2 - \omega_b$, $l_4$, $l_5$, and $l_6\in \mathbb{R}$ are positive gain scalars.
	
	\begin{theorem}
		\label{theorem_rot}
		Consider system (\ref{rot_dynamics}) under the rotational controller proposed in Eq. (\ref{rot_controller}) and the fixed-time SMO in Eq. (\ref{rot_SMO_controller}). With Assumption \ref{assump_derivative}, the rotational closed-loop system is input-to-state stable and $\hat d_{\tau}$ can converge to $d_{\tau}$ uniformly in fixed time when the following conditions hold:
		\begin{subequations}
			\begin{numcases}{}
				\Theta \left(\bm{R}_b(0), \bm{R}_b^d(0)\right) < 2, \label{theorem2_1}\\
				\Vert e_\omega (0) \Vert^2 < \frac{2}{\lambda_{max}(J)}K_R\left( 2 - \Theta \left(\bm{R}_b(0), \bm{R}_b^d(0)\right) \right), \label{theorem2_2}\\
				l_4 > \sqrt{2l_6},l_5>0,l_6>4\Vert J^{-1}\Vert\rho_2.
			\end{numcases}
		\end{subequations}
		where $\Theta \left(\bm{R}_b, \bm{R}_b^d\right) = \frac{1}{2} tr\left(\bm{I}_3 - \bm{R}_b^{d\top}\bm{R}_b\right)$ and $\lambda_{max}(\cdot)$ represents the maximum eigenvalue of a matrix.
	\end{theorem}
	
	The detailed stability analysis of the rotational loop can be found in Appendix \ref{App_B}.

	\section{Simulation and Experiments}
	\label{sec:Simulation and Experiments}
	
	This section presents a series of simulations and experiments to evaluate the proposed methodology. Comparative simulations demonstrate the superiority of the proposed method in terms of computational efficiency and accuracy in handling incomplete information dynamic games. Subsequently, the experimental setup and implementation details are introduced. Several trajectory tracking experiments are conducted to validate the control accuracy which is essential for precise target capture. Finally, a real-world experiment confirms the practical effectiveness of our approach in non-cooperative target capture task. The experiments can be viewed in the attached videos.
	
	\subsection{Comparative Simulations}
	\label{comp_sim}
	
	\begin{figure*}[!t]
		\setlength{\abovecaptionskip}{-0.3cm}
		\centering{\includegraphics[width=\textwidth]{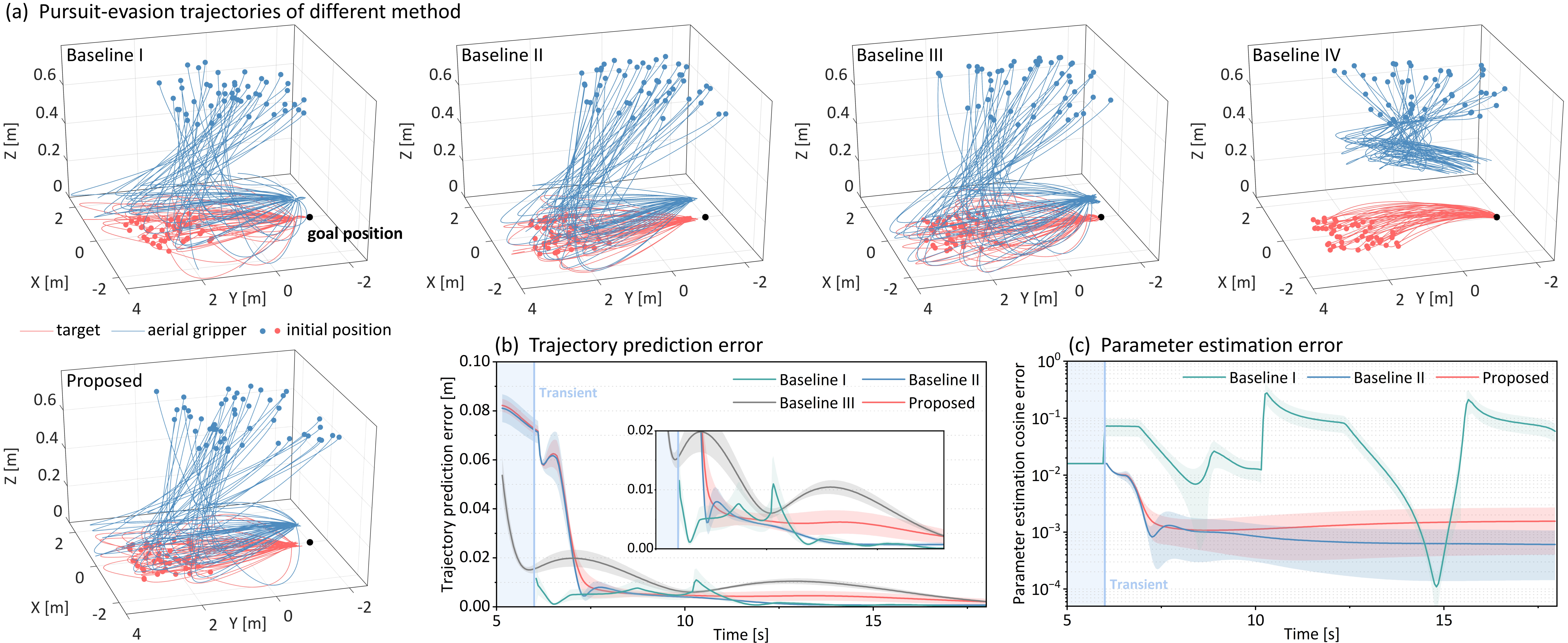}}
		\caption{Comparative simulation results. \textbf{(a)} Pursuit-evasion trajectories of different methods. More red lines that reach the goal position means a lower success rate. \textbf{(b)} Trajectory prediction error. The light blue area indicates the transient stage in which the cost parameters have not been estimated since the construction of initial trajectory data takes time.\textbf{(c)} Parameter estimation error.}
		\label{Comparative Simulation}
	\end{figure*}
	
	To validate the effectiveness and advantage of the proposed GRIPPER framework, comprehensive Monte Carlo simulations are conducted. The initial positions of the aerial gripper and the target, as well as the goal position of the target, are uniformly randomized across broad designated regions. The non-cooperative target is commanded to evade the pursuit of the aerial gripper while advancing toward the goal position $p_{\C{T},\text{goal}}$. The target motion is generated by a separate forward-game planner with full knowledge of the aerial gripper. During the simulations, the average speed of the evasive target was 0.37 m/s, with the maximum speed reaching 1.21 m/s, demonstrating highly dynamic and adversarial interactions.
	
	To isolate the specific benefits of the proposed computationally efficient inverse-game solver, four baselines are evaluated:
	
	\begin{itemize}
		
		\item \textbf{Baseline I (KKT-based Solver)\cite{petersOnlineOfflineLearning2023}}: A KKT-constraint-based inverse-game solver, by which the unknown parameters and the Nash equilibrium are jointly inferred through the KKT conditions.
		
		\item \textbf{Baseline II (Gradient-based Solver \cite{liuLearningPlayTrajectory2023})}: A state-of-the-art gradient-based inverse-game solver that leverages IFT but explicitly computes the Hessian inversion to generate exact gradient.
		
		\item \textbf{Baseline III (NI-MPC)}: A non-interactive Model Predictive Control (NI-MPC) baseline that predicts the target moving at a constant velocity.
		
		\item \textbf{Baseline IV (PID Guidance)}: A purely reactive, non-game-theoretic PID approaching method.
		
	\end{itemize}
	
	The forward game of the proposed method and Baseline II, as well as Baseline I and Baseline III are all solved by an open-source solver CasADi (via Matlab API) on the computer equipped with Intel i5-1035G7 CPU.
	
	Two indices are utilized to quantify the simulation results, including a cosine dissimilarity for parameter estimation performance and mean absolute error for trajectory prediction performance. The parameter estimation error is defined as
	\begin{equation}
		\label{indices_estimation}
		\gamma(\lambda_{\C{T}}, \hat\lambda_{\C{T}}) = 1-\frac{\lambda_{\C{T}}\T \hat\lambda_{\C{T}}}{\Vert\lambda_{\C{T}}\Vert_2\Vert\hat\lambda_{\C{T}}\Vert_2},
	\end{equation}
	and the single-step trajectory prediction error is formulated as
	\begin{equation}
		\label{indices_prediction}
		\Gamma(p_{\C{T}}, p_{\C{T}}^{\text{pred}}) = \frac{1}{\tilde N}\sum_{t=1}^{\tilde N} \Vert p_{\C{T}, t} - p_{\C{T}, t}^{\text{pred}} \Vert,
	\end{equation}
	where $\tilde N$ is the prediction horizon, $p_{\mathcal{T}, t}$ denotes the trajectory sequence of the non-cooperative target, obtained by sampling its actual position with time interval of the system model $\Delta t$.
	
	The comprehensive statistical results including computation time, trajectory prediction error, parameter estimation error, capture success rate, and mean capture time, are summarized in Table \ref{Table_comparative_sim}. The resulting pursuit-evasion trajectories are visualized in Fig. \ref{Comparative Simulation}(a).
	
	\begin{table}[!t]
		\centering
		\renewcommand{\arraystretch}{1.2}
		\caption{Comparative Simulation Results}
		\setlength{\tabcolsep}{4pt}
		\begin{tabular}{ccccccc}
			\hline
			Method & \makecell{Comp. \\time (ms)} & \makecell{Pred. \\err. (mm)} & \makecell{Est. err.\\ ($10^{-3}$)}  & \makecell{Succ. \\rate} & \makecell{Capt. \\time (s)} & Horizon \\ \hline
			\textbf{Proposed}                & \textbf{19.67} & \textbf{8.07} & \textbf{2.07} & \textbf{100\%} & \textbf{9.28} & \textbf{20} \\
			Baseline I     & 78.71         & 2.93          & 58.47         & 96\%           & 10.70         & 20          \\
			Baseline II                   & 28.55          & 7.19          & 1.59          & 100\%          & 9.63          & 20          \\
			Baseline III                  & 15.26          & 11.15         & -             & 58\%           & 12.90         & 20          \\
			Baseline IV                   & $<$ 1          & -             & -             & 0\%            & -             & -           \\ \hline
		\end{tabular}
		\label{Table_comparative_sim}
	\end{table}
	
	\subsubsection{Comparison against inverse-game solvers (Baseline I \& II)} With the same effective prediction horizon of 20 steps, Baseline I achieves a low trajectory prediction error of 2.93 mm and a capture success rate of 96\%. However, the parameter estimation exhibits a relatively large error and noticeable fluctuations, as shown in Fig. \ref{Comparative Simulation}(c). In contrast, the proposed method achieves a low estimation error within 2 s and maintains stable convergence. More critically, the high-dimensional joint optimization problem results in an average single-step computation time of 78.71 ms ($<$ 15 Hz). Such severe computational latency may not be feasible for dynamic aerial capture task.
	
	\begin{remark}
		A counterintuitive result is that Baseline I achieves a low trajectory prediction error despite large parameter estimation error. This occurs because the joint optimization problem directly minimizes the trajectory residual, while recovery of the nominal cost weights is indirect. Moreover, different cost vectors may generate similar short-horizon equilibrium trajectories \cite{liCostInferenceFeedback2023}. Thus, these two errors should not be interpreted as equivalent performance indicators. Nevertheless, accurate cost weight estimation remains valuable for behavioral interpretation and prediction under interactions.
	\end{remark}
	
	Baseline II shares the same underlying implicit differentiation philosophy as the proposed method and achieves almost identical task performance (100\% success rate as well as comparable prediction error and estimation error) as shown in Fig. \ref{Comparative Simulation}(b)-(c). However, since Baseline II explicitly inverts the Hessian matrix, its single-step computation time is noticeably slower than the proposed method. With the aid of our method, the single-step computation time is reduced by 75\% compared to Baseline I and by 31\% compared to Baseline II.
	
	\subsubsection{Comparison against non-game-theoretic methods (Baseline III \& IV)}
	To isolate the indispensable role of game-theoretic reasoning and interactive prediction, GRIPPER is evaluated against two non-game-theoretic baselines. The PID guidance law is formulated as:
	$$
	u_{\mathcal{G}}(t) = K_1 (p_{\mathcal{T}}(t) - p_{\mathcal{G}}(t)) + K_2 (v_{\mathcal{T}}(t) - v_{\mathcal{G}}(t)).
	$$
	where $K_1, K_2 \in \mathbb{R}^{3\times 3}$ are positive diagonal gain matrices.
	
	As depicted in Fig. \ref{Comparative Simulation}(a), Baseline IV fails to accomplish the capture task. Due to the lack of trajectory prediction, the pursuer always lagged behind the evasive maneuvers of the target, failing to secure the required proximity for grasping.
	
	\begin{figure*}[!t]
		\centering{\includegraphics[width=0.85\textwidth]{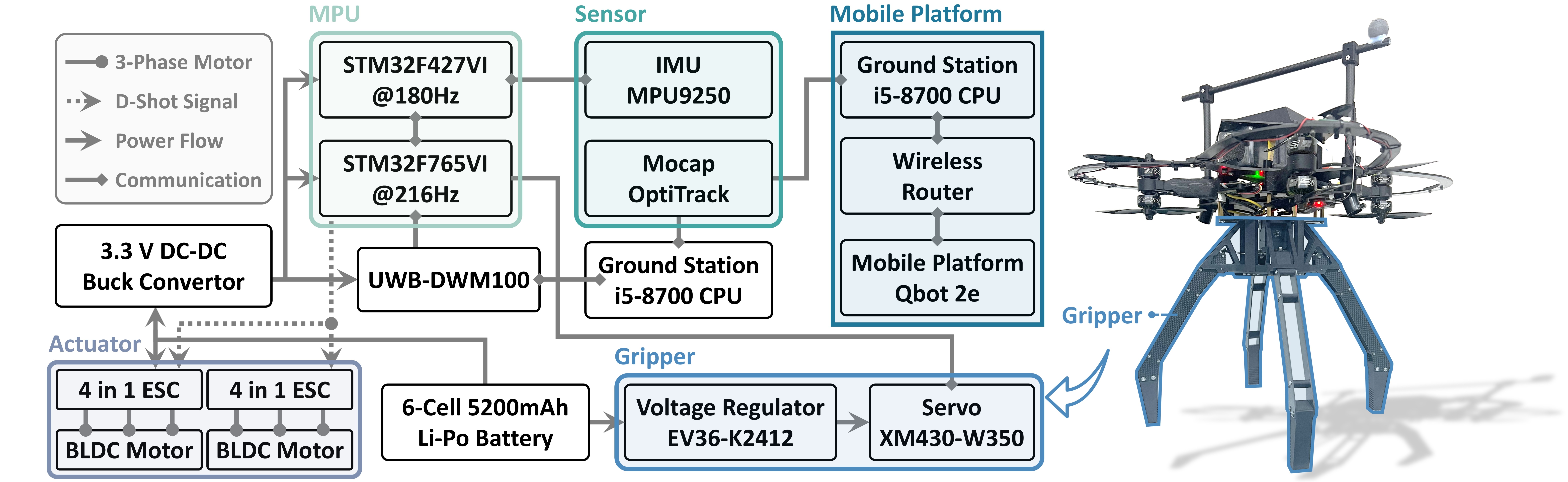}}
		\caption{Hardware structure of the experiment system.}
		\label{Hardware}
	\end{figure*}
	
	Baseline III involves predictive planning but predicts the target moving at constant velocity. While Baseline III exhibits excellent computational speed (15.26 ms), it suffers from severe prediction degradation. As shown in Fig. \ref{Comparative Simulation}(a), the target motion appears approximately straight late in the approach. However, the target trajectory is generated by the forward-game planner and still exhibits nonzero acceleration. During the 2 s preceding the minimum horizontal distance between the aerial gripper and the target, the average target acceleration remains 0.16 m/s\textsuperscript{2}. NI-MPC cannot anticipate these velocity variations, causing prediction errors to accumulate over the horizon. As shown in Fig. \ref{Comparative Simulation}(b), the prediction error is significantly higher than those of the other methods before 15 s. At around 17 s, the target have decelerated and constant-velocity prediction becomes locally more accurate, allowing successful capture in some cases. Consequently, NI-MPC achieves a lower capture success rate of 58\% and a longer mean capture time of 12.90 s.
	
	In summary, simpler tracking methods fail due to a lack of foresight, while naive predictive methods fail due to an inability to anticipate adversarial reactivity. The comprehensive simulations illustrate that the proposed GRIPPER framework achieves superior computational efficiency and remarkable adaptability for non-cooperative target capture tasks. These advantages support real-world deployment and enable proactive, game-driven dynamic aerial capture.
	
	\subsection{Hardware Setup}
	
	As shown in Fig. \ref{Hardware}, this experiment employs two robotic systems: the aerial gripper and the Quanser QBot 2e Mobile Platform as the non-cooperative target.
	
	The autopilot comprises two microprocessors (STM32F4 and STM32F7) and an inertial measurement unit (MPU9250). For accurate localization, the position data is provided by a motion capture system (OptiTrack) with sub-millimeter-level accuracy. The position of the UAV expressed in the world frame is received by the ground station and transmitted to the UAV via an ultra-wideband communication module (DWM100). The attitude is estimated onboard using the IMU, where a complementary filter is implemented on the STM32F4 to fuse accelerometer and gyroscope data. Two 4-in-1 electronic speed controllers (Hobbywing) are used to control the brushless DC motors (Xing 2208 1800KV), with control outputs represented by DShot signals at a frequency of 500 Hz. The gripper is actuated by a DYNAMIXEL XM430-W350 servo, weighing 82g with a stall torque of 4.1 N·m. Power is supplied by a 6S 5200 mAh Li-Po battery, with a voltage regulator to ensure stable operating voltage for the servo actuator.
	
	The software program consists of two portions: control module and planning module. The UAV control algorithm is executed on STM32F7, with translational and rotational controllers running at 100 Hz and 500 Hz, respectively. The codes inside the microcontroller unit (MCU) are developed and debugged on MATLAB with C language. The servo actuator has a built-in PID controller and receives the desired joint angle from the STM32F7. The prediction and planning module, including the game-theoretic solver, is implemented on the ground station equipped with Intel i7-8700 CPU. The forward-game is solved online using CasADi (via Matlab API). Position of the ground vehicle is tracked by the motion capture system, and the velocity is estimated through a Luenberger observer.
	
	The ground vehicle is controlled on a separate ground station that performs control algorithm and sends differential drive signals to the platform through WiFi. A separate forward-game planner is deployed to generate the desired trajectory of the ground vehicle.
	
	\subsection{Experiment I: Trajectory Tracking}
	\label{exp_Trajectory_Tracking}
	
	To evaluate the effectiveness of the proposed anti-disturbance controller, two trajectory tracking tests are conducted. Besides the baseline controller without disturbance compensation (Eq. (\ref{trans_baseline_controller}) and Eq. (\ref{rot_baseline_controller})), a traditional nonlinear DO based controller \cite{chenNonlinearDisturbanceObserver2000} is chosen for comparison as it exhibits decent estimation performance for slowly varying disturbance. The details are provided in Appendix \ref{NDO}.
	
	To quantitatively evaluate the experiment results, three indices (mean absolute error (MAE) $\nu_j$, standard deviation (STD) $\sigma_j$, and maximum error $\kappa_j$) are defined as
	\begin{equation}
		\begin{aligned}
			\label{indices_controller}
			\nu_j & = \frac{1}{n} \sum_{i=1}^{n}\Vert p_j^d(i) - p_j(i) \Vert, \\
			\sigma_j & = \sqrt{\frac{1}{n-1} \sum_{i=1}^{n}\left( \Vert p_j^d(i) - p_j(i) \Vert-\nu_j \right)^2}, \\
			\kappa_j & = \max\limits_{1\leq i\leq n}\left(\Vert p_j^d(i) - p_j(i) \Vert\right),
		\end{aligned}
	\end{equation}
	where $n$ is the length of data set collected during flights. $j \in \{b, e\}$ represents the indices for body frame and the gripper frame, respectively. The operator $\Vert\cdot\Vert$ denotes the standard Euclidean norm for vectors. The quantitative results are listed in Table \ref{Controller_Exp_Table}
	
	\vspace{8pt}
	\noindent \textbf{Test I: Robustness to unknown payload}
	\vspace{8pt}
	
	Test I mainly evaluates the robustness and recovery speed of the proposed controller during the grasping and releasing of objects with unknown mass. In this test, the aerial gripper is commanded to grasp objects with different masses (260g, 560g, and 830g) respectively, maneuver in the desired trajectory and release the object halfway, as illustrated in Fig. \ref{Controller_Exp}(a). A lemniscate trajectory is carried out which is formulated as:
	\begin{equation}
		p_e^d =
		\begin{bmatrix}
			r_x\sin s \\
			r_y\sin (2s)\\
			-z_0 + r_z \left(-\frac{1}{2} + \frac{1}{2} \cos (2s)\right)
		\end{bmatrix},
	\end{equation}
	where $s$ is the trajectory parameter with respect to $t$. To ensure the smoothness of the trajectory, $s$ is designed as the integral of a three-stage speed function (Appendix \ref{design_traj}). In this experiment, the trajectory parameters are set as $r_x=2$ m, $r_y=1$ m, $r_z=0.5$ m, and $z_0=0.5$ m. The desired trajectory of the UAV platform $p_b^d$ is obtained by adopting differential flatness in Section \ref{PathPlan}. The aerial gripper releases the object when $s = 4\pi$.
	
	\begin{table}[!t]
		\centering
		\renewcommand{\arraystretch}{1.2}
		\setlength{\tabcolsep}{4pt}
		\caption{comparison results of three control methods (UNIT:$\mathrm{cm}$)}
		\begin{tabular}{ccccc}
			\hline
			Description   & Method   & $\nu_e$ & $\sigma_e$ & $\kappa_e$ \\ \hline
			\multirow{3}{*}{\begin{tabular}[c]{@{}c@{}}Test I\\ 260g payload\end{tabular}}
			& Proposed & \textbf{3.49}		  & \textbf{1.74}		 & \textbf{12.76}		 \\
			& Baseline & 7.61 $\uparrow$ 54\% & 2.85 $\uparrow$ 39\% & 16.02 $\uparrow$ 20\% \\
			& DO-based & 5.35 $\uparrow$ 35\% & 2.56 $\uparrow$ 32\% & 15.26 $\uparrow$ 16\% \\ \hline
			\multirow{3}{*}{\begin{tabular}[c]{@{}c@{}}Test I\\ 560g payload\end{tabular}}
			& Proposed & \textbf{3.58}		  & \textbf{2.46}		 & \textbf{14.15}		 \\
			& Baseline & 8.84 $\uparrow$ 60\% & 4.91 $\uparrow$ 50\% & 27.07 $\uparrow$ 48\% \\
			& DO-based & 7.03 $\uparrow$ 49\% & 3.71 $\uparrow$ 34\% & 21.12 $\uparrow$ 33\% \\ \hline
			\multirow{3}{*}{\begin{tabular}[c]{@{}c@{}}Test I\\ 830g payload\end{tabular}}
			& Proposed & \textbf{5.35}		  & \textbf{4.24}		 & \textbf{25.81}        \\
			& Baseline &10.19 $\uparrow$ 47\% & 7.26 $\uparrow$ 42\% & 35.94 $\uparrow$ 28\% \\
			& DO-based & 7.79 $\uparrow$ 31\% & 4.59 $\uparrow$ 8\%  & 26.97 $\uparrow$ 4\%  \\ \hline
			\multirow{3}{*}{\begin{tabular}[c]{@{}c@{}}Test II\\ Gripper actuation\end{tabular}}
			& Proposed & \textbf{6.19}		  & \textbf{2.66}		 & \textbf{14.92}        \\
			& Baseline & 9.26 $\uparrow$ 33\% & 3.02 $\uparrow$ 12\% & 19.20 $\uparrow$ 22\% \\
			& DO-based & 8.68 $\uparrow$ 29\% & 3.68 $\uparrow$ 28\% & 34.10 $\uparrow$ 56\% \\ \hline
		\end{tabular}
		\label{Controller_Exp_Table}
	\end{table}
	
	\begin{figure*}[htp]
		\setlength{\abovecaptionskip}{-0.3cm}
		\centering{\includegraphics[width=1\textwidth]{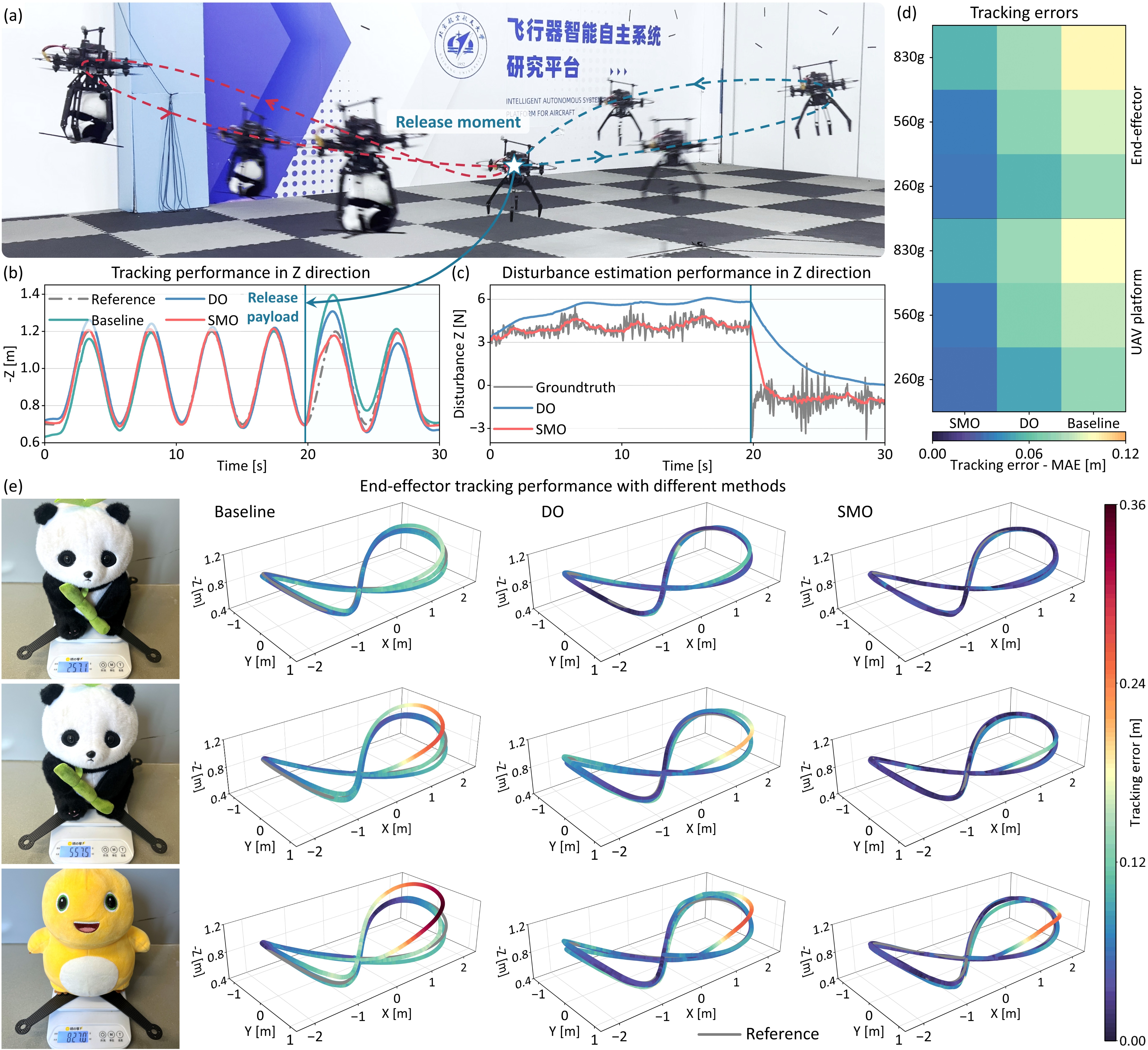}}
		\caption{Experimental results of the aerial gripper's control performance. \textbf{(a)} Details of Test I in Section \ref{exp_Trajectory_Tracking}. \textbf{(b)} Tracking details in the Z-axis direction. The shaded area represents the tracking results after releasing the grasped payload. \textbf{(c)} Disturbance estimation performance of different methods in the Z-axis direction. 'Groundtruth' denotes the truth disturbance calculated by noise-corrupted acceleration measurements offline only for post-analysis. \textbf{(d)} Tracking errors of the UAV and the end-effector under different payloads. \textbf{(e)} Tracking performance under different methods and different payloads. The color bar indicates the tracking error evaluated by standard Euclidean-distance (Eq. (\ref{indices_controller})).}
		\label{Controller_Exp}
	\end{figure*}
	
	\begin{figure*}[htp]
		\setlength{\abovecaptionskip}{-0.3cm}
		\centering{\includegraphics[width=1\textwidth]{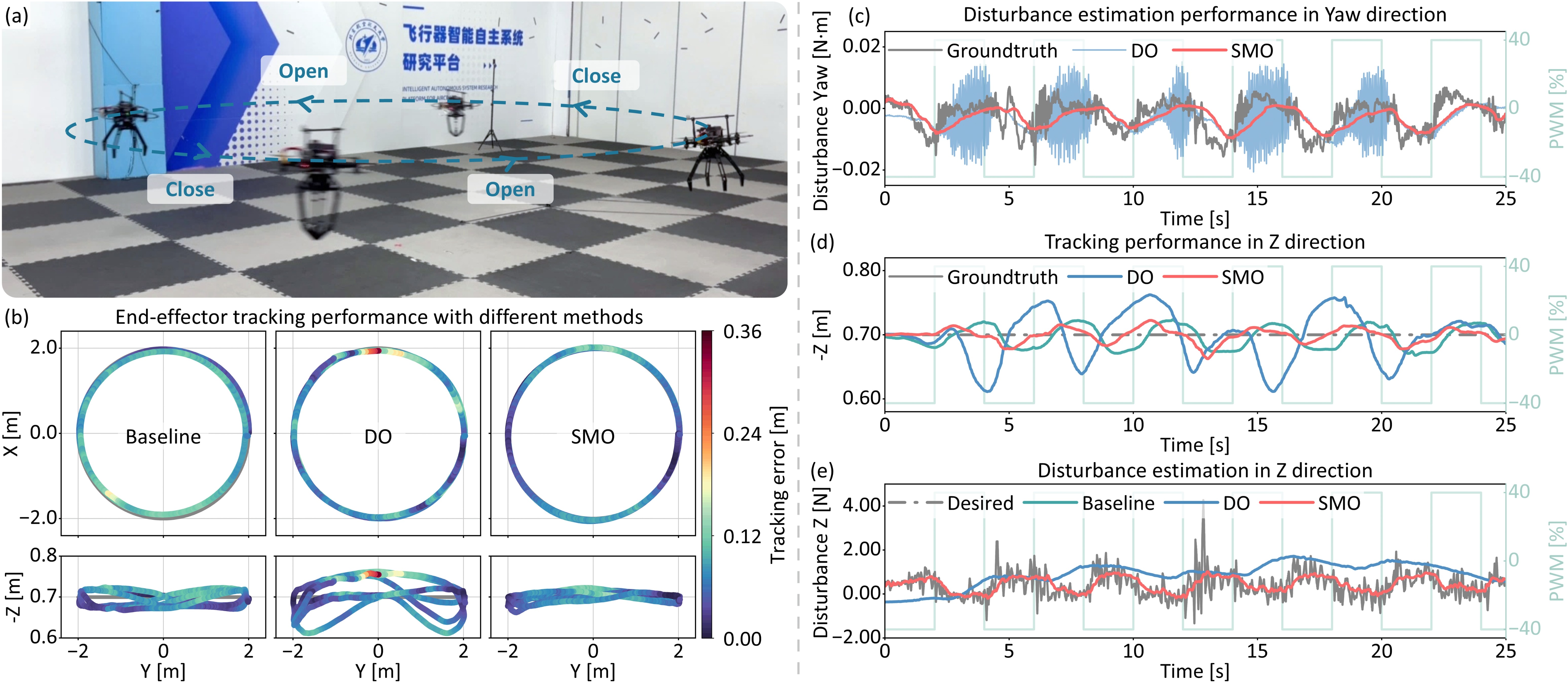}}
		\caption{Experimental results of the aerial gripper's control performance. \textbf{(a)} Details of Test II in Section \ref{exp_Trajectory_Tracking}. \textbf{(b)} Tracking performance under different methods. \textbf{(c)} Torque disturbance estimation performance of different methods in the yaw direction. The right Y-axis represents the PWM signal of the gripper servo. Gripper opens when PWM = -40\% and closes when PWM = 40\%. \textbf{(d)} Tracking details in the Z-axis direction. \textbf{(e)} Force disturbance estimation performance of different methods in the Z-axis direction.}
		\label{Controller_Exp_test2}
	\end{figure*}
	
	As shown in Fig. \ref{Controller_Exp}(c), accurate estimation of disturbances caused by unknown payloads is achieved by proposed control scheme, leading to the lowest tracking error (Fig. \ref{Controller_Exp}(b),(d),(e)). Notably, the estimation and tracking performance reveals robust at the moment of payload release. As illustrated in Fig. \ref{Controller_Exp}(d),(e) and Table \ref{Controller_Exp_Table}, the control performance of all control schemes deteriorates to varying extents as the payload mass increases. However, under the disturbances of all given payloads, the proposed SMO-based controller consistently demonstrates superior tracking accuracy compared to the alternatives. Moreover, as illustrated in Fig. \ref{Controller_Exp}(b),(d), the proposed controller achieves faster recovery in trajectory tracking after the payload release. Experimental results highlight the robustness and superiority of the proposed control scheme in handling disturbances caused by payload with unknown mass.
	
	\vspace{8pt}
	\noindent \textbf{Test II: Robustness to gripper actuation}
	\vspace{8pt}
	
	In the following test, the disturbances caused by gripper actuation is considered. The aerial gripper is commanded to repeatedly open and close its gripper while tracking a circular trajectory, as depicted in Fig. \ref{Controller_Exp_test2}(a). The trajectory is formulated as $p_e^d = \begin{bmatrix} r \sin s & r \cos s &  -z_1 \end{bmatrix}\T$, where $s$ is the trajectory parameter with respect to $t$, and in the experiment $r=2$ m, $z_0=0.7$ m.
	
	It can be seen from Fig. \ref{Controller_Exp_test2}(b)-(e) that the gripper actuation mainly induces force and torque disturbances along z-axis. This is primarily due to the torque disturbances generated by the actuation of the vertically installed servo. Additionally, the gripper actuation induces COM and moment of inertia shift along z-axis, while the impact on moment of inertia along x and y axes remains minor. This phenomenon is consistent with previous analysis. Notably, as illustrated in Fig. \ref{Controller_Exp_test2}(b),(d), the trajectory tracking performance of the aerial gripper deteriorates when disturbance estimation and compensation are performed using traditional DO method. Moreover, as shown in Fig. \ref{Controller_Exp_test2}(c),(e), the DO-based controller fails to provide accurate disturbance estimation in this experiment and exhibits significant oscillations in yaw direction. Traditional DO method is not well-suited for handling such kind of disturbances. In contrast, the proposed SMO-based controller demonstrates superior performance, as evidenced by the significantly reduced tracking error shown in Table \ref{Controller_Exp_Table}.
	
	Consequently, from the experiments in both tests, it can be concluded that accurate and robust trajectory tracking performance is achieved by adopting the proposed SMO-based control scheme. Prompt disturbance rejection is guaranteed, thereby providing a solid control foundation for successful capture of non-cooperative targets.
	
	\subsection{Experiment II: Non-Cooperative Target Capture}
	\label{exp_Target_Capture}
	
	Finally, in an attempt to further explore the capability of the proposed scheme in achieving non-cooperative target capture tasks and high-precision trajectory prediction, a real-world experiment is carried out. The experimental setup is almost the same as the comparative simulation. The initial position of the aerial gripper and the non-cooperative target is set to $\begin{bmatrix} -1.5,0,-0.7 \end{bmatrix} \mathrm{m}$ and $\begin{bmatrix} 0,-1.5,-0.3 \end{bmatrix} \mathrm{m}$, respectively. The scenario is depicted in Fig. \ref{Game_Exp}(a). After grasping, the aerial gripper is commanded to deliver the target to the specific position ($\begin{bmatrix} -1,2,0 \end{bmatrix} \mathrm{m}$) and release it. Unlike in simulation, the inverse-game parameters require additional fine-tuning in real-world experiments. The final parameter settings are summarized in Table \ref{GameParameters}. 
	
	This experimental protocol allows for the assessment of comprehensive abilities of the aerial gripper, including computational efficiency, parameter estimation, trajectory prediction, path planning, and control accuracy. A set of 5 experiments is conducted, out of which the aerial gripper successfully completes the task 5 times. The key snaps together with an experimental video are illustrated in Fig. \ref{Scenario1}(b). As shown in Fig. \ref{Game_Exp}(c), the non-cooperative target is successfully captured, demonstrating satisfactory performance in both planning and control. The target maximum speed reaches 0.25 m/s during the pursuit-evasion game. Furthermore, as illustrated in Fig. \ref{Game_Exp}(e), due to the superior performance of the proposed controller, the aerial gripper maintains reliable control during the post-capture delivery phase. This capability is of practical significance for ensuring reliable execution of subsequent tasks after capture of non-cooperative target.
	
	Notably, as shown in Fig. \ref{Game_Exp}(c), around $t = 12$ s, the aerial gripper is in front of the non-cooperative target. This is partly attributed to the trajectory prediction capability of the game-theoretic planner. On the other hand, the ground vehicle fails to reach the expected position on time due to limitations in its controller performance. However, since the inverse-game-theoretic module estimates unknown parameters based on the actual trajectory of the target, this mismatch can be compensated to some extent. Consequently, the planner adapts in a short time and enables successful capture.
	
	The estimation performance of the unknown cost weight vector could directly influence the accuracy of trajectory prediction and the performance of pursuit trajectory planning. As shown in Fig. \ref{Game_Exp}(b), the estimation converges within 4 s in the experiment. This duration accounts for the necessary initial data buffering phase and the observations required to smooth out real-world measurement noise.
	
	In terms of trajectory prediction error, two baselines are taken for comparison: a forward game solver without parameter estimation and a forward game solver using the ground truth target weights, which are referred to as "No estimation" and "Groundtruth" in Fig. \ref{Game_Exp}(d), respectively. The GRIPPER framework initially exhibits a significant prediction error but rapidly achieves a prediction accuracy comparable to the "Groundtruth" as the parameter estimation converges, outperforming the non-adaptive baseline. GRIPPER can rapidly enhance its prediction capability by effectively leveraging information from the non-cooperative target.
	
	Notice that in real-world experiment, the target trajectory actually deviates from the theoretical Nash equilibrium due to control errors, actuator noise, and real-world disturbances. Despite this mismatch, the proposed method can identify meaningful cost parameters and allow effective pursuit. Therefore, the framework is robust to moderate deviations from the assumed model structure.
	
	Due to the limited communication bandwidth between the ground station and the UAV, the game-theoretic planner is configured to update at a fixed frequency of 20 Hz. This frequency is sufficient for the aerial gripper to ensure smooth trajectory tracking. In conclusion, aiming at dynamic capture of non-cooperative target, the proposed incomplete information dynamic game solver meets the real-time requirements of the aerial gripper system.
	
	\begin{figure*}[!t]
		\centering{\includegraphics[width=0.95\textwidth]{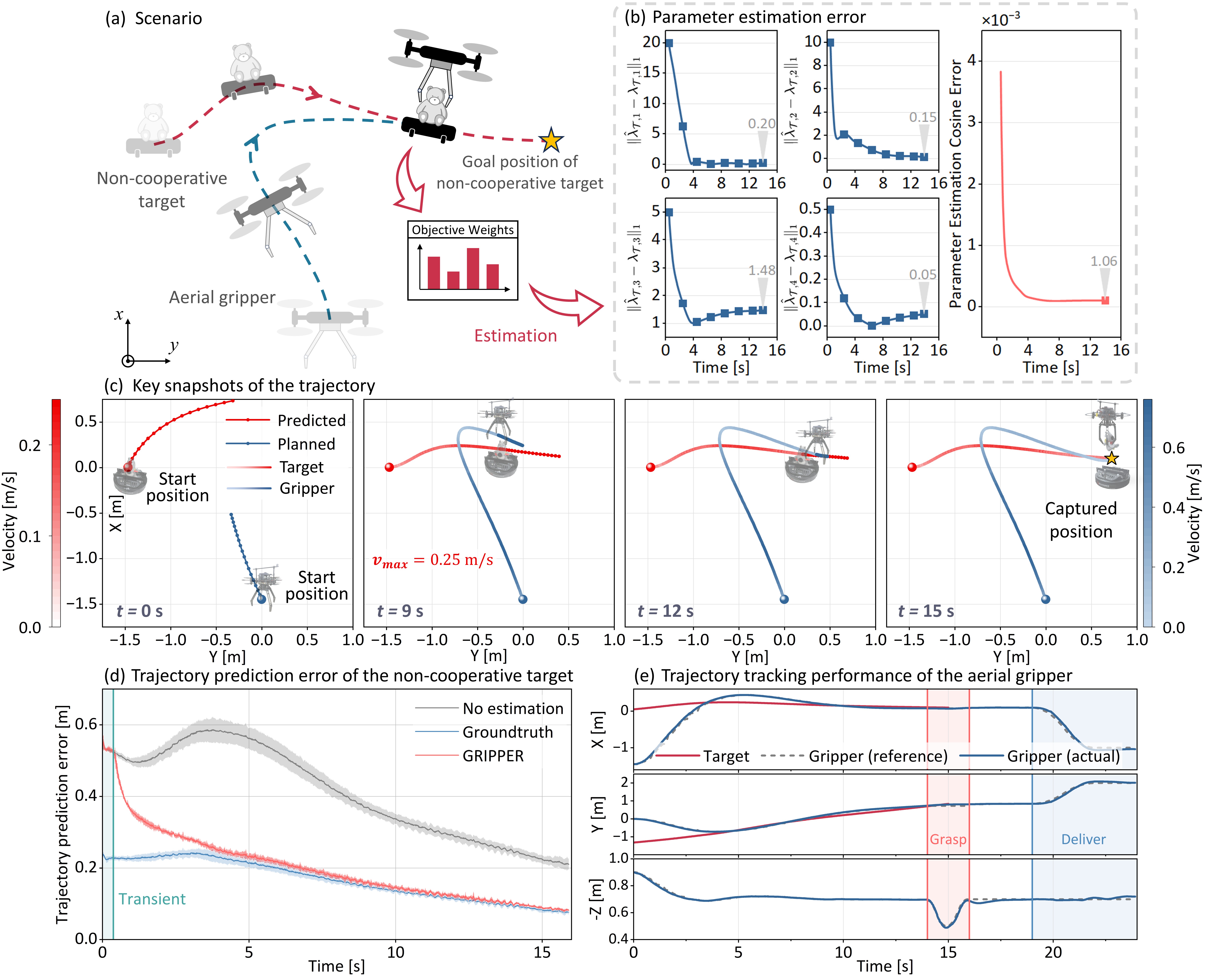}}
		\caption{Experimental results of non-cooperative target capture task. \textbf{(a)} The scenario of Experiment II in Section \ref{exp_Target_Capture}. \textbf{(b)} Estimation performance of unknown cost weight vector $\lambda_{\C{T}}$. \textbf{(c)} Key snapshots of the trajectory in non-cooperative target capture task. \textbf{(d)} Trajectory prediction error of the non-cooperative target. The green area indicates the transient stage in which the cost parameters have not been estimated since the construction of initial trajectory data takes time. \textbf{(e)} Trajectory of the target and the aerial gripper.}
		\label{Game_Exp}
	\end{figure*}
	
	\begin{table}[!t]
		\centering
		\renewcommand{\arraystretch}{1.2}
		\caption{Parameters of GRIPPER}
		\begin{tabular}{ccc}
			\hline
			Description  & Parameter               & Value                               \\ \hline
			Cost weight vector
			& $\lambda_{\C{T}}$       & $[30,10,1]\T$                       \\
			& $\lambda_{\C{G}}$       & $[100,10,10,1]\T$                   \\ \hline
			Initial guess of $\lambda_{\C{T}}$                                                        
			& $\hat\lambda_{\C{T},0}$ & $[120,20,5,0.5]\T$                  \\ \hline
			\multirow{2}{*}{Learning rate}                                                            
			& $\alpha$                & $0.2\times \mathbf{I}_{3N\times 3N}$    \\
			& $\beta$                 & $\mathrm{diag}([4,0.8,0.05,0.001])$ \\ \hline
		\end{tabular}
	\label{GameParameters}
\end{table}

\subsection{Grasping Execution and Conditions for Successful Capture}

As described in Section III-B, the four fingers are actuated by a servo through four symmetric linkages. In the experiments, the opening and closing of the gripper are directly controlled by preset constant PWM commands sent from the STM32F7 microprocessors to the servo.

During approach, the gripper remains open and a small vertical offset is maintained to avoid unintended contact with the target. Once the horizontal distance between the gripper center and the target falls below the capture tolerance, a predefined descent-closure-ascent sequence is triggered. During descent, the closing PWM signal is applied so that the four fingers enclose the target. The closing signal is then maintained while the UAV ascends with the captured target.

A successful dynamic aerial capture is governed by both algorithmic prerequisites and geometric constraints. The receding-horizon solver requires the target state to be observable at a rate exceeding the online planning frequency. However, there are no strict deterministic numerical thresholds for prediction accuracy or inference quality that guarantee a capture. Ultimately, a reasonable condition for a successful capture is dictated by the geometric tolerance of the physical hardware and the target size.

The custom aerial manipulator features a 4-finger gripper. Due to the physical gaps between the fingers, if the horizontal alignment error exceeds a specific threshold, the target will slip through the side gaps during closure. Assume that the target is a rigid cylinder of radius $r_{\C{T}}$, and the horizontal distance between the gripper center and the target center is $d$. It can be deduced from the geometric relationship that the condition $d < r_{\C{T}}$ is supposed to be strictly satisfied for a successful capture at the moment of closure.

Therefore, the performance of our framework lies in its ability to guide the aerial gripper into this r-radius neighborhood before the target escapes. In our simulations and real-world experiments, this geometric threshold is set as 0.05 m.

\section{Conclusion}
\label{sec:Conclusion}

In this work, a gradient-based real-time inverse-game prediction and planning framework GRIPPER is presented for the aerial gripper to achieve non-cooperative target capture task. The interaction between the aerial gripper and the target is modeled as a general-sum PE game under incomplete information. Real-time intent inference and trajectory prediction are achieved by efficiently solving a bilevel optimization problem with a gradient-based method. Furthermore, a fixed-time SMO-based controller is developed to guarantee robust tracking performance under gripper-induced and payload-related disturbances. The proposed control strategy enables fast recovery and stable trajectory tracking, which is essential for successful grasping in highly dynamic settings.

Extensive simulations and hardware experiments have been conducted to validate the proposed framework. GRIPPER demonstrates computational efficiency and inference accuracy. The capacity of the aerial gripper equipped with GRIPPER and the anti-disturbance controller has been well verified by the real-world non-cooperative target capture task.

\section{Discussion and Future Work}
\label{sec:discussion}

Despite the promising results of the proposed GRIPPER framework, there exist certain theoretical boundaries that present promising directions for future research. From a game-theoretic perspective, the current framework isolates incomplete information to the pursuer by modeling the target as a fully informed adversary. Mutual inference may provide a more comprehensive representation of competitive settings. Furthermore, the framework relies on predefined basis functions and assumes constant objective parameters, which implies that the target behavior is sufficiently predictable. The receding-horizon implementation exhibits practical robustness to moderate structural misspecifications and slow parameter variations. However, the system may face challenges with abrupt evasive maneuvers, sudden strategic shifts, and control delays. In future work, more uncertainty-aware inference mechanisms and learning-based approaches will be investigated for more assumption-free cost representations.

Additionally, extending the framework to jointly infer the unknown dynamics remains a compelling objective to better handle fully unknown adversaries. In our current system, the dynamics model is treated as a known prior to ensure strict objective identifiability. Overcoming these challenges will be a key focus for our future research in less time-critical application scenarios.

	\section*{Appendix}
	\label{App}
	
	\renewcommand{\thesubsection}{\Alph{subsection}}
	\renewcommand{\theequation}{A\arabic{equation}}
	\setcounter{equation}{0}
	\renewcommand{\thelemma}{A\arabic{lemma}}
	\setcounter{lemma}{0}
	\renewcommand{\theHequation}{appendix.A.\arabic{equation}}
	
	\subsection{Stability Analysis of Translational Controller}
	\label{App_A}
	
	\subsubsection{Convergence Analysis of the Fixed-Time SMO}

	\begin{lemma}
		Consider the fixed-time SMO in Eq. (\ref{trans_SMO_controller}). $\hat d_f$ can uniformly converge to $d_f$ in fixed time, under Assumption \ref{assump_derivative} and the conditions
		\begin{equation}
			\label{condition1}
			l_1 > \sqrt{2l_3},l_2>0,l_3>4m^{-1}\rho_1.
		\end{equation}
	\end{lemma}
	
	\begin{proof}
		With the definition $y_1 = \varepsilon_2 - \frac{1}{m}d_f$, the error dynamics of the observer can be obtained by taking derivative of $z_1$ and $y_1$ as
		\begin{equation}
			\left\{
			\begin{aligned}
				& \dot z_1 = -l_1\frac{z_1}{\Vert z_1 \Vert^{1/2}} - l_2 z_1 \Vert z_1 \Vert + y_1 \\
				& \dot y_1 = -l_3\frac{z_1}{\Vert z_1 \Vert} + \phi_1
			\end{aligned}
			\right.,
		\end{equation}
		where $\phi_1 = - \frac{1}{m} \dot d_f$. According to Assumption \ref{assump_derivative}, it holds that $\Vert \phi_1 \Vert \le m^{-1}\rho_1$.
		
		On the basis of results in \cite{basinMultivariableContinuousFixedtime2017} and the conditions Eq. (\ref{condition1}), $z_1$ and $y_1$ can uniformly converge to the origin within fixed time as
		\begin{equation}
			t_1 \le \left( \frac{1}{l_2 \delta_1} + \frac{2(\sqrt{3} \delta_1)^{1/2}}{l_1} \right) \left( 1 + \frac{M_1}{M_2(1-\sqrt{2l_3}/l_1)} \right),
		\end{equation}
		where $\delta_1 > 0$, $M_1 = l_3 + m^{-1}\rho_1$, and $M_2 = l_3 - m^{-1}\rho_1$. 
		
	\end{proof}

	\subsubsection{Stability Analysis of the Closed-loop System}
	
	\begin{proof}[Proof of Theorem \ref{theorem_trans}]
		By combining Eq. (\ref{trans_controller}) and Eq. (\ref{trans_dynamics}), the closed-loop error dynamics is obtained as
		\begin{equation}
			\dot e = \underbrace{\begin{bmatrix}
					\mathbf{0}_{3\times 3} & \mathbf{I}_{3\times 3} \\ -K_p & -K_v\end{bmatrix}}_{A_e}e 
			+ \underbrace{\begin{bmatrix}
					\mathbf{0}_{3\times 3} & \mathbf{0}_{3\times 3} \\ 
					\mathbf{0}_{3\times 3} & \frac{1}{m}\mathbf{I}_{3\times 3} \end{bmatrix}}_{B_e}
			\begin{bmatrix} \bm{0}_3 \\ \tilde d_f \end{bmatrix},
		\end{equation}
		where $e = \begin{bmatrix} e_p\T & e_v\T \end{bmatrix}\T$ and $\tilde{d}_f = \hat{d}_f - d_f$ represents the estimation error of the disturbance $d_f$. It can be checked that the eigenvalues of $A_e$ all have negative real parts as long as $K_p$ and $K_v$ have positive diagonal elements. As a consequent, there must exist a positive definite symmetric matrix $P\in \mathbb{R}^{6\times 6}$ that meets $A_e\T P + PA_e = -I$. Consider the following Lyapunov function
		\begin{equation}
			V_1 = e\T P e.
		\end{equation}
		Taking derivative of $V_1$ yields
		\begin{equation}
			\label{Lyapunov1}
			\begin{aligned}
				\dot V_1 & =  \dot e\T Pe + e\T P\dot e \\
				& = e\T(PA_e+A_e\T P)e + 2e\T PB_e \begin{bmatrix} \bm{0}_3 \\ \tilde d_f \end{bmatrix}.
			\end{aligned}
		\end{equation}
		By resorting to Young's inequality, it can be verified that
		\begin{equation}
			e\T PB_e \begin{bmatrix} \bm{0}_3 \\ \tilde d_f \end{bmatrix}
			\le \frac{1}{2\varsigma}\lambda_{max}^2(PB_e)e\T e + \frac{\varsigma}{2}\tilde{d}\T_f \tilde{d}_f,
		\end{equation}
		where $\varsigma$ is an arbitrary positive constant and $\lambda_{max}(\cdot)$ represents the maximum eigenvalue of a matrix. Thus, Eq. (\ref{Lyapunov1}) can be thereby written as
		\begin{equation}
			\dot V_1 \le -e\T e + \frac{1}{\varsigma}\lambda_{max}^2(PB_e)e\T e + \varsigma\tilde{d}\T_f \tilde{d}_f.
		\end{equation}
		Given that $e\T e \le -\frac{e\T P e}{\lambda_{max}(P)}$, one can obtain that
		\begin{equation}
			\dot V_1 \le - \eta_1 e\T P e  + \varsigma\tilde{d}\T_f \tilde{d}_f,
		\end{equation}
		where $\eta_1 = \frac{\varsigma - \lambda_{max}^2(PB_e)}{\varsigma\lambda_{max}(W)}$. The condition $\varsigma > \lambda_{max}^2(PB_e)$ holding can ensure $\eta_1$ positive.
		
		According to the convergence analysis of the observer in \cite{basinMultivariableContinuousFixedtime2017}, the convergence of $\tilde{d}_f$ is independent of the stability of the closed-loop system error $e$. Thus, according to the results in \cite{basinMultivariableContinuousFixedtime2017}, $\Vert y_1 \Vert$ is bounded by
		\begin{equation}
			\begin{aligned}
				& \Vert y_1(t) \Vert < \Upsilon_1 = M_1 \left( \frac{1}{l_2 \delta_1} + \frac{2(\sqrt{3} \delta_1)^{1/2}}{l_1} \right),\\
				& t > t_2 = \left( \frac{1}{l_2 \delta_1} + \frac{2(\sqrt{3} \delta_1)^{1/2}}{l_1} \right).
			\end{aligned}
		\end{equation}
		Given that $ y_1 = m^{-1}\tilde{d}_f$, one can obtained that
		\begin{equation}
			\dot V_1(t) < - \eta_1 V_1 + \varsigma m^2\Upsilon_1^2, \quad t > t_2,
		\end{equation}
		which implies that
		\begin{equation}
			\label{Lyapunov1_final}
			0 \le V_1(t) < \left(V_1(t_2) - \frac{\varsigma m^2\Upsilon_1^2}{\eta_1}\right)e^{-\eta_1 (t-t_2)} + \frac{\varsigma m^2\Upsilon_1^2}{\eta_1}.
		\end{equation}
		
		Finally, according to Eq. (\ref{Lyapunov1_final}), all signals in the translational loop are globally uniformly bounded.

	\end{proof}

	\subsection{Stability Analysis of Rotational Controller}
	\label{App_B}
	
	\subsubsection{Convergence Analysis of the Fixed-Time SMO}
	
	\begin{lemma}
		Consider the fixed-time SMO in Eq. (\ref{rot_SMO_controller}). $\hat d_f$ can uniformly converge to $d_f$ in fixed time, under Assumption \ref{assump_derivative} and the conditions
		\begin{equation}
			\label{condition2}
			l_4 > \sqrt{2l_6},l_5>0,l_6>4\Vert J^{-1}\Vert\rho_2.
		\end{equation}
	\end{lemma}
	\begin{proof}
		With the definition $y_2 = \varepsilon_4 - J^{-1}d_f$, the error dynamics of the observer can be obtained by taking derivative of $z_1$ and $y_1$ as
		\begin{equation}
			\left\{
			\begin{aligned}
				& \dot z_2 = -l_4\frac{z_2}{\Vert z_2 \Vert^{1/2}} - l_5 z_2 \Vert z_2 \Vert + y_2 \\
				& \dot y_2 = -l_6\frac{z_2}{\Vert z_2 \Vert} + \phi_2
			\end{aligned}
			\right.,
		\end{equation}
		where $\phi_2 = - J^{-1} \dot d_f$. According to Assumption \ref{assump_derivative}, it holds that $\Vert \phi_2 \Vert \le \Vert J^{-1}\Vert \rho_2$.
		
		On the basis of results in \cite{basinMultivariableContinuousFixedtime2017} and the conditions Eq. (\ref{condition2}), $z_2$ and $y_2$ can uniformly converge to the origin within fixed time as
		\begin{equation}
			t_3 \le \left( \frac{1}{l_5 \delta_2} + \frac{2(\sqrt{3} \delta_2)^{1/2}}{l_4} \right) \left( 1 + \frac{M_3}{M_4(1-\sqrt{2l_6}/l_4)} \right),
		\end{equation}
		where $\delta_2 > 0$, $M_3 = l_6 + \Vert J^{-1}\Vert\rho_2$, and $M_4 = l_6 - \Vert J^{-1}\Vert\rho_2$. 
		
	\end{proof}

	\subsubsection{Stability Analysis of the Closed-loop System}
	
	\begin{proof}[Proof of Theorem \ref{theorem_rot}]
		The derivative of $e_\omega$ is given as
		\begin{equation}
			\dot e_\omega = - \omega_b^{\times} \bm{R}_{b}\T \bm{R}_{b}^{d} \omega_b^d + \bm{R}_{b}\T \bm{R}_{b}^{d} \dot \omega_b^d - \dot \omega_b.
		\end{equation}
		By combining Eq. (\ref{rot_controller}) and Eq. (\ref{rot_dynamics}), the closed-loop error dynamics is obtained as
		\begin{equation}
			\label{omega_error_dynamics}
			J \dot e_\omega = - K_R e_R - K_\omega e_\omega + \tilde{d}_\tau.
		\end{equation}
		
		By resorting to the property of skew-symmetric operator, the derivative of $e_R$ can be obtained as
		\begin{equation}
			\begin{aligned}
				\dot e_R & = \frac{1}{2}\left( tr(\bm{R}_{b}\T \bm{R}_{b}^{d})\mathbf{I}_{3\times 3} - \bm{R}_{b}\T \bm{R}_{b}^{d} \right) e_\omega \\
				& = \Psi(\bm{R}_{b}\T \bm{R}_{b}^{d})e_\omega
			\end{aligned}
		\end{equation}
		Based on Rodrigues' formula \cite{bulloGeometricControlMechanical2005}, $\Vert \Psi(\bm{R}_{b}\T \bm{R}_{b}^{d}) \Vert \le 1$ for any $\bm{R}_{b}\T \bm{R}_{b}^{d} \in SO(3)$.
		
		By taking derivative of $\Theta \left(\bm{R}_b, \bm{R}_b^d\right)$, the attitude error dynamic is expressed as
		\begin{equation}
			\label{attitude_error_dynamics}
			\dot \Theta \left(\bm{R}_b, \bm{R}_b^d\right) = e_\omega\T e_R.
		\end{equation}
		
		Consider the following Lyapunov function
		\begin{equation}
			V_2 = \frac{1}{2}e_\omega\T J e_\omega + K_R \Theta \left(\bm{R}_b, \bm{R}_b^d\right) + ce_R\T e_\omega,
		\end{equation}
		where $c$ is a non-negative constant. From Eqs. (\ref{omega_error_dynamics}-\ref{attitude_error_dynamics}), the time derivative of $V_2$ is given as
		\begin{equation}
			\label{Lyapunov2}
			\begin{aligned}
				\dot V_2 & = e_\omega\T J \dot e_\omega + K_R \dot \Theta \left(\bm{R}_b, \bm{R}_b^d\right) + c(\dot e_R\T e_\omega + e_R\T \dot e_\omega) \\
				& = -K_\omega e_\omega\T e_\omega + c e_\omega\T \Psi\T(\bm{R}_{b}\T \bm{R}_{b}^{d}) e_\omega  \\
				& \quad + ce_R\T J^{-1}(-K_R e_R - K_\omega e_\omega) \\
				& \quad + e_\omega\T \tilde{d}_\tau + ce_R\T J^{-1}\tilde{d}_\tau\\
				& = \C{V} + e_\omega\T \tilde{d}_\tau + ce_R\T J^{-1}\tilde{d}_\tau.
			\end{aligned}
		\end{equation}
		
		Leveraging the methodology delineated in \cite{jiaFlightControlQuadrotor2022}, one can obtain that
		\begin{equation}
			\C{V} \le -\Lambda\T Q_1 \Lambda,
		\end{equation}
		where $\Lambda = \begin{bmatrix} \Vert e_R \Vert & \Vert e_\omega \Vert \end{bmatrix}\T$, and $Q_1$ is
		\begin{equation}
			Q_1 = \begin{bmatrix}
				\frac{cK_R}{\lambda_{max}(J)} & -\frac{cK_\omega}{2\lambda_{min}(J)} \\
				-\frac{cK_R}{2\lambda_{min}(J)} & K_\omega - c
			\end{bmatrix}
		\end{equation}
		where $\lambda_{min}$ represents the minimum eigenvalue of a matrix.
		
		By accounting for the conditions Eq. (\ref{theorem2_1}) and Eq. (\ref{theorem2_2}), it can be obtained that
		\begin{equation}
			\begin{gathered}
				\Lambda\T \underline{Q_1} \Lambda \le V_2 \le \Lambda\T \overline{Q_1} \Lambda, \\
				\underline{Q_1} = \frac{1}{2}\begin{bmatrix}
					K_R & -c \\ -c & \lambda_{min}(J)
				\end{bmatrix}, \\
				\overline{Q_1} = \frac{1}{2}\begin{bmatrix}
					\frac{2K_R}{2-\underline{V_2}(0)/K_R} & c \\ c & \lambda_{max}(J)
				\end{bmatrix},
			\end{gathered}
		\end{equation}
		where $\underline{V_2} = V_2|_{c=0}=\frac{1}{2}e_\omega\T J e_\omega + K_R \Theta \left(\bm{R}_b, \bm{R}_b^d\right)$ and satisfies $0 \le \underline{V_2}(0)/K_R < 2$. To make sure that $Q_1$, $\underline{Q_1}$, and $\overline{Q_1}$ are all positive definite matrices, $c$ can be selected as
		\begin{equation}
			c\le \min\{ K_\omega, \sqrt{K_R\lambda_{min}(J)}, \frac{4K_\omega K_R \lambda_{min}^2(J)}{K_\omega^2 \lambda_{max}(J)+4K_R\lambda_{min}^2(J)} \}.
		\end{equation}
		
		Using Young's inequality, it can be achieved that
		\begin{equation}
			\begin{gathered}
				e_\omega\T \tilde{d}_\tau \le e_\omega\T e_\omega + \frac{1}{4}\tilde{d}_\tau\T\tilde{d}_\tau, \\
				e_R\T J^{-1}\tilde{d}_\tau \le\frac{1}{\lambda_{max}(J)}e_R\T e_R + \frac{\lambda_{max}(J)}{4\lambda_{min}^2(J)}\tilde{d}_\tau\T\tilde{d}_\tau.
			\end{gathered}
		\end{equation}
		
		Thus, Eq. (\ref{Lyapunov2}) can be thereby written as
		\begin{equation}
			\begin{gathered}
				\dot V_2 \le -\Lambda\T Q_2 \Lambda + \varpi \tilde{d}_\tau\T\tilde{d}_\tau, \\
				Q_2 = \begin{bmatrix}
					c\frac{K_R-1}{\lambda_{max}(J)} & -\frac{cK_\omega}{2\lambda_{min}(J)} \\
					-\frac{cK_R}{2\lambda_{min}(J)} & K_\omega - c - 1
				\end{bmatrix}, \\
				\varpi = \frac{\lambda_{min}^2(J) + c\lambda_{max}(J)}{4\lambda_{min}^2(J)}.
			\end{gathered}
		\end{equation}
		
		Furthermore, in attempt to make sure that $Q_2$ is a positive definite matrix, $c$ can be chosen as
		\begin{equation}
			\begin{aligned}
				c \le & \min\{ K_\omega-1, \sqrt{K_R\lambda_{min}(J)}, \\
				& \frac{4(K_\omega-1) (K_R-1) \lambda_{min}^2(J)}{K_\omega^2 \lambda_{max}(J)+4(K_R-1)\lambda_{min}^2(J)} \}.
			\end{aligned}
		\end{equation}
		
		According to the convergence analysis of the observer in \cite{basinMultivariableContinuousFixedtime2017}, the convergence of $\tilde{d}_\tau$ is independent of the stability of the closed-loop system error $e_R$ and $e_\omega$. Thus, according to the results in \cite{basinMultivariableContinuousFixedtime2017}, $\Vert y_2 \Vert$ is bounded by
		\begin{equation}
			\begin{aligned}
				& \Vert y_2(t) \Vert < \Upsilon_2 = M_3 \left( \frac{1}{l_5 \delta_2} + \frac{2(\sqrt{3} \delta_2)^{1/2}}{l_4} \right),\\
				& t > t_4 = \left( \frac{1}{l_5 \delta_2} + \frac{2(\sqrt{3} \delta_2)^{1/2}}{l_4} \right).
			\end{aligned}
		\end{equation}
		Given that $ y_2 = J^{-1}\tilde{d}_f$, one can obtained that
		\begin{equation}
			\dot V_2(t) < - \eta_2 V_2 + \Xi, \quad t > t_4,
		\end{equation}
		where $\eta_2 = \frac{\lambda_{min}(Q_2)}{\lambda_{max}(\overline{Q_1})}$ and $\Xi = \varpi \lambda_{max}^2(J) \Upsilon_2^2$. Therefore, it can be given that
		\begin{equation}
			\label{Lyapunov2_final}
			0 \le V_2(t) < \left(V_2(t_4) - \frac{\Xi}{\eta_2}\right)e^{-\eta_2 (t-t_4)} + \frac{\Xi}{\eta_2}.
		\end{equation}
		
		Finally, according to Eq. (\ref{Lyapunov2_final}), all signals in the rotational loop are globally uniformly bounded.
		
	\end{proof}

	\subsection{Nonlinear Disturbance Observer}
	\label{NDO}
	
	In Experiment I (Section \ref{exp_Trajectory_Tracking}), as a performance measure, a classical nonlinear disturbance observer is taken for comparison. The translational DO is formulated as
	\begin{equation}
		\left\{	
		\begin{aligned}
			& \dot \zeta_1 = -L_1 \zeta_1 - L_1(L_1 v_b + gw_3 - \frac{f_b}{m} b_3) \\
			& \hat d_f = m(\zeta_1 + L_1 v_b) \\
		\end{aligned}
		\right.,
	\end{equation}
	where $\zeta_1 \in \mathbb{R}^3$ is an auxiliary variable, $L_1 \in \mathbb{R}^{3\times 3}$ is positive diagonal gain matrix.
	
	The rotational DO is formulated as
	\begin{equation}
		\left\{	
		\begin{aligned}
			& \dot \zeta_2 = -L_2 \zeta_2 - L_2( L_2 \omega_b - J^{-1}\omega_b^{\times}J\omega_b + J^{-1}\tau_b) \\
			& \hat d_{\tau} = J(\zeta_2 + L_2 \omega_b)
		\end{aligned}
		\right.,
	\end{equation}
	where $\zeta_2 \in \mathbb{R}^3$ is an auxiliary variable, $L_2 \in \mathbb{R}^{3\times 3}$ is positive diagonal gain matrix.
	
	The baseline controller is consistent with the proposed control scheme.

	\subsection{Design of Desired Trajectory}
	\label{design_traj}
	
	In Experiment I (Section \ref{exp_Trajectory_Tracking}), to ensure trajectory smoothness, the parameter $s$ is designed as the integral of a three-stage speed function
	\begin{equation}
		\dot{s}(t) =
		\begin{cases}
			v_{\max} \cdot q(\tfrac{t}{t_a}), & t \in [0, t_a] \\
			v_{\max}, & t \in [t_a, T - t_a] \\
			v_{\max} \cdot q(\tfrac{T-t}{t_a}), & t \in [T - t_a, T]
		\end{cases}
		,
	\end{equation}
	where $q(x) = 10x^3 - 15x^4 + 6x^5$ is a smooth polynomial function ensuring zero acceleration at boundaries, $v_{\max} \in \mathbb{R}$ is the maximum speed, $t_a$ is the acceleration duration, and $T$ is the total duration.

	\bibliography{ref}
	\bibliographystyle{IEEEtran}

\end{document}